\documentclass[onecolumn,
  draftclsnofoot,
  11pt]{IEEEtran}
\usepackage{amsmath,amsfonts}
\usepackage{algorithmic}
\usepackage{algorithm}
\usepackage{array}
\usepackage[caption=false,font=normalsize,labelfont=sf,textfont=sf]{subfig}
\usepackage{textcomp}
\usepackage{stfloats}
\usepackage{url}
\usepackage{verbatim}
\usepackage{graphicx}
\usepackage{cite}
\usepackage{enumitem}
\usepackage{amssymb}
\usepackage{multirow}
\usepackage{mathnotation}  

\usepackage[hidelinks]{hyperref}

\begin{document}
\title{Lean Clients, Full Accuracy: Hybrid \\Zeroth- and First-Order Split Federated Learning}

\author{
Zhoubin Kou, Zihan Chen, Jing Yang, Cong Shen
\thanks{The authors are with the \textit{Charles L. Brown} Department of Electrical and Computer Engineering, University of Virginia, USA. (E-mail: \texttt{\{zhoubin, brf3rx, yangjing, cong\}@virginia.edu}.)}
}

\markboth{}%
{Shell \MakeLowercase{\textit{et al.}}: A Sample Article Using IEEEtran.cls for IEEE Journals}

\maketitle

\begin{abstract}
Split Federated Learning (SFL) enables collaborative training between resource-constrained edge devices and a compute-rich server. Communication overhead is a central issue in SFL and can be mitigated with auxiliary networks. Yet, the fundamental client-side computation challenge remains, as back-propagation requires substantial memory and computation costs, severely limiting the scale of models that edge devices can support. To enable more resource-efficient client computation and reduce the client-server communication, we propose HERON-SFL, a novel hybrid optimization framework that integrates zeroth-order (ZO) optimization for local client training while retaining first-order (FO) optimization on the server.
With the assistance of auxiliary networks, ZO updates enable clients to approximate local gradients using perturbed forward-only evaluations per step, eliminating memory-intensive activation caching and avoiding explicit gradient computation in the traditional training process.
Leveraging the low effective rank assumption, we theoretically prove that HERON-SFL's convergence rate is \emph{independent} of model dimensionality, addressing a key scalability concern common to ZO algorithms.
Empirically, on ResNet training and language model (LM) fine-tuning tasks, HERON-SFL matches benchmark accuracy while reducing client peak memory by up to 64\% and client-side compute cost by up to 33\% per step, substantially expanding the range of models that can be trained or adapted on resource-limited devices.

\end{abstract}


\section{Introduction}
\IEEEPARstart{D}{istributed} machine learning (DML), such as Federated Learning (FL) and Split Learning (SL), is critical for privacy-sensitive domains like healthcare \cite{vepakomma2018split} and IoT \cite{wu2020collaborate}, where data cannot be centrally shared \cite{kairouz2021advances}. While FL enables parallel training \cite{konevcny2015federated, mcmahan2017communication}, it requires full model deployment on clients, burdening resource-constrained devices and risking model privacy. Conversely, SL \cite{gupta2018distributed} reduces client-side computation via model partitioning but suffers from sequential, relay-based training that limits throughput at scale. To bridge this gap, Split Federated Learning (SFL) \cite{thapa2022splitfed} has emerged as a paradigm particularly attractive for edge–cloud systems. By integrating FL's parallelism with SL's resource efficiency, it enables scalable collaborative training while maintaining data locality and privacy.

Specifically, SFL partitions the neural network into a client-side frontend and a server-side backend. During training, clients transmit intermediate activations (smashed data) to the server and wait for gradients, enabling privacy-preserving collaboration. Recent SFL work equips each client with an auxiliary network (typically a lightweight output layer) that estimates the cut-layer gradients locally \cite{ han2021accelerating, oh2022locfedmix}. This design decouples the client from the server, allowing the client sub-model to update immediately without waiting for the server’s backward pass and avoiding gradient transmission from the server side~\cite{mu2025federated, nair2025fsl}.
However, this decoupling typically increases client-side complexity: \emph{clients must also train the auxiliary module with backpropagation}, which requires caching intermediate activations whose footprint grows with model depth/width and is further amplified by the auxiliary branch; for modern deep convolutional neural network (CNN) and Transformers \cite{vaswani2017attention}, activation memory can dominate the on-device footprint and exceed tight edge-device RAM budgets~\cite{cai2020tinytl,korthikanti2023reducing}.
This raises a critically important question with substantial practical implications: \textit{can we preserve the decoupled, auxiliary-network SFL communication pattern while reducing client-side training overhead (backpropagation and activation caching), without sacrificing model performance?}

We answer this question affirmatively by proposing {HERON-SFL}, a Hybrid zEroth- and fiRst-Order (FO) optimizatioN training framework that performs zeroth-order (ZO) optimization \cite{liu2020primer} on the client side and retains FO on the server side. Unlike FO, ZO gradient-free updates are based on a few perturbed forward evaluations and thus drastically reduce the computation load. However, ZO methods are often criticized for slow and high-variance updates in high dimensions. 
Fortunately, under the SFL with auxiliary network setting, ZO optimization is confined to the shallow, low-dimensional client-side submodel, allowing HERON-SFL to leverage the memory- and computation-efficiency of zeroth-order updates while avoiding their typical slow and high-variance behavior from affecting the overall model training \cite{chen2023layer}.
Under a low effective rank condition \cite{malladi2023fine}, we theoretically prove that HERON-SFL admits \emph{dimension-independent} convergence, and the experiments show HERON-SFL's accuracy is comparable to FO baselines while greatly reducing the memory and computation load.
Meanwhile, the perturbations in ZO are fully local and incur \textit{no} additional uploads of smashed activations beyond those already required for server-side FO training. As a result, the server receives smashed activations at the same frequency as in auxiliary-network-based FO SFL.
Our main contributions are summarized as follows: 
\begin{itemize}[leftmargin=10pt]
    \item We propose {HERON-SFL}, a novel hybrid zeroth- and first-order SFL framework. 
    Building upon an auxiliary network that enables decoupled local training, we introduce zeroth-order (ZO) optimization on the client-side. This eliminates backpropagation for local updates, thereby greatly reducing on-device memory and computational costs.
    
    \item We provide the first theoretical study of hybrid ZO–FO optimization in SFL. Our analysis shows that under a low effective rank assumption, which is supported by empirical evidence, HERON-SFL achieves an $\mathcal{O}(1/\sqrt{T})$ convergence rate, which matches the convergence order of FO SFL analyses under comparable assumptions. 
    
    \item We conduct comprehensive experiments spanning both vision (ResNet training) and language (LM fine-tuning) tasks. Results show that HERON-SFL consistently reduces client peak memory by up to 64\% and client computation per step by up to 33\%, while matching the accuracy of state-of-the-art, auxiliary-network-based FO SFL methods. These gains highlight HERON-SFL’s algorithm-level superiority and practicality, enabling advanced models on previously infeasible devices without increasing communication beyond auxiliary-network-based FO SFL.
\end{itemize}

\section{Related Works}
\subsection{Split Federated Learning}
While modern foundation models achieve state-of-the-art performance \cite{brown2020language, chowdhery2023palm}, their immense computational and memory requirements restrict them to data centers, limiting their real-world reach \cite{ sani2024future}. SFL was proposed by merging Federated Learning (FL) \cite{mcmahan2017communication} with Split Learning (SL) \cite{vepakomma2018split} to enhance data privacy and robustness \cite{thapa2022splitfed, lee2024exploring}. 
However, standard SFL entails prohibitive communication overhead due to the transmission of high-dimensional intermediate data \cite{vepakomma2018split}.
Recent work has explored communication-efficient strategies specifically under the SFL architecture, primarily targeting the transmission of smashed data. By leveraging activation compression \cite{shiranthika2024splitfedzip}, quantization \cite{zhang2024federated}, sparsification \cite{zheng2023reducing}, or adaptive upload policies \cite{liu2022wireless}, these methods seek to reduce communication overhead while preserving the SL workflow.

Nevertheless, reducing the data payload does not resolve the inherent limitations of the SL paradigm. The synchronous update lock persists, as clients must strictly await gradients from the server before updating, leading to unavoidable idle time \cite{kairouz2021advances}. To address this, {\it algorithmic decoupling} aims to eliminate the synchronous lock by incorporating a client-side auxiliary model. This approach decouples client and server updates by generating local gradient estimates, thereby obviating the need for server-to-client gradient transmission \cite{han2021accelerating, mu2025federated, oh2022locfedmix, nair2025fsl}. Inspired by decoupled training in centralized settings \cite{belilovsky2019greedy}, this strategy can significantly reduce communication latency. However, these auxiliary models introduce a significant trade-off: a substantial increase in the client's computational and memory footprint, as the auxiliary network can be considerably larger than the primary client-side model itself \cite{nair2025fsl}.

Concurrently, to address the heterogeneity and resource constraints of real-world edge environments, recent research has pursued {\it system-level optimization}. These methods adapt the SFL protocol to varying network and hardware conditions. Key contributions include adaptive model splitting based on network status \cite{lin2024adaptsfl}, hierarchical topologies for managing client resources \cite{lin2025hierarchical}, parallel training designs optimized for wireless networks \cite{wu2023split}, and dynamic resource-based tiers to accelerate training under heterogeneous environments \cite{mohammadabadi2024speed}. 
However, these system-centric designs inevitably introduce additional implementation complexity, often requiring sophisticated orchestration to manage dynamic topologies or splitting points based on specific heterogeneity assumptions. In contrast, this work targets resource efficiency from an \textit{algorithmic perspective}, aiming to construct a streamlined SFL framework that inherently minimizes resource consumption without relying on complex system-level coordination.

\subsection{Zeroth-Order Optimization for DML}
ZO optimization estimates gradients through function evaluations \cite{liu2020primer}, making it particularly useful when explicit gradients are unavailable, such as in reinforcement learning \cite{nakashima2025unifying, zhang2024zeroth} and privacy-sensitive scenarios \cite{chen2017zoo, liu2018zeroth, liu2019signsgd}. 
Recently, ZO has gained attention as an efficient strategy for training \cite{chen2024deepzero} and fine-tuning \cite{malladi2023fine}, since it avoids back-propagation's memory and compute overhead. 
In distributed machine learning, ZO has been explored as a gradient estimator in FL, demonstrating promising benefits in privacy preservation \cite{zhang2021desirable, fang2022communication, ling2024convergence} and communication reduction \cite{li2024achieving}. 
However, its adoption in the SFL framework remains limited. 
The main barrier is that variance reduction in ZO requires multiple perturbations, which would substantially increase intermediate activation transmissions and thus communication overhead. 
To address this, we \textit{restrict ZO to the client-side} with the help of auxiliary networks, enabling resource-efficient training while avoiding additional communication costs.

\section{Preliminaries}
\subsection{SFL with Auxiliary Network}
We consider a \textit{synchronous} SFL system with one server and $N$ clients, each holding a private dataset $\mathcal{D}_i$, where the entire dataset is the set $\{\mathcal{D}_i\}_{i=1}^{N}$.
The global model is split at a cut layer into client- and server-side sub-models, where we denote the collection of parameters as $\boldsymbol{\theta}_g=\{\boldsymbol{\theta}_c,\boldsymbol{\theta}_s\}$. 
Each client $i$ owns a local version of the client-side model, $\boldsymbol{\theta}_{c,i}$. 
For a sample $\xi_{i, j}\in\mathcal{D}_i$, client $i$ performs a forward pass up to the cut layer to produce the smashed data, $\boldsymbol s_i=\boldsymbol{\theta}_{c,i}(\xi_{i, j})$, and uploads it to the Main-Server. 
The server then completes the forward pass by processing these activations with its sub-model $\boldsymbol{\theta}_s$.
The goal is to minimize the global loss function:
\begin{equation}
\begin{aligned}
\min_{\boldsymbol{\theta}_c,\boldsymbol{\theta}_s}\ f(\boldsymbol{\theta}_g)
=\frac{1}{N}\sum\nolimits_{i=1}^N f_i(\boldsymbol{\theta}_g)
= \frac{1}{N}\sum\nolimits_{i=1}^N \mathbb{E}_{\xi_{i, j}\sim\mathcal D_i}\!\left[\ell(\boldsymbol{\theta}_g;\xi_{i, j})\right],
\end{aligned}
\end{equation}
where $f_i(\boldsymbol{\theta}_g)$ and $f(\boldsymbol{\theta}_g)$ measure the expected loss on the global model over client $i$'s local dataset $\mathcal{D}_i$ and the entire dataset, respectively, computed using a task-specific loss function $\ell(\cdot)$ (e.g., cross-entropy).

We adopt the single Main-Server style framework\footnote{Recent studies CSE-FSL~\cite{mu2025federated} and FSL-SAGE~\cite{nair2025fsl} report that auxiliary-network SFL based on the SFLV2-style protocol can yield substantially improved performance compared with SFLV1-style variants.}: a single server-side model $\boldsymbol{\theta}_s$ resides on the Main-Server and is trained by sequentially processing smashed data $\boldsymbol s_i$ from all clients, while a Fed-Server aggregates client-side parameters into the average
$\bar{\boldsymbol{\theta}}_c := \frac{1}{N}\sum\nolimits_{i} \boldsymbol{\theta}_{c,i}$ (initial parameters for the next round). 
To reduce communication overhead and enable client-side local feedback, each client $i$ attaches an auxiliary (Aux) model $\boldsymbol{\theta}_{a,i}$ to form a local predictor $\boldsymbol{\theta}_{l,i}(\xi_{i, j})
= \boldsymbol{\theta}_{a,i}\!\left(\boldsymbol{\theta}_{c,i}(\xi_{i, j})\right)$, where $\boldsymbol{\theta}_{l,i}=\{\boldsymbol{\theta}_{c,i},\boldsymbol{\theta}_{a,i}\}$~\cite{mu2025federated, oh2022locfedmix}.
Because of the Aux model, the SFL system breaks the \textit{training lock} between $\boldsymbol{\theta}_c$ and $\boldsymbol{\theta}_s$: by leveraging $\boldsymbol{\theta}_a$, the client can perform local updates independently of server‐side gradient feedback.

After initializing the global model $\{\boldsymbol{\theta}_c,\boldsymbol{\theta}_s\}$, the basic SFL-Aux algorithm proceeds: 
in each round, client $i$ computes smashed data $\mathcal S_{i}=\boldsymbol{\theta}_{c,i}(\xi_{i})$ on local mini-batches $\xi_{i} = \{\xi_{i, j}\}_{j=1}^B$ and uploads them to the Main-Server, while updating $\boldsymbol{\theta}_{l,i}$ by minimizing a local loss from $\boldsymbol{\theta}_{a,i}(\mathcal S_{i})$, with backpropagation confined to the client.
The Main-Server queues smashed data from all clients and sequentially executes forward/backward passes to update $\boldsymbol{\theta}_s$. 
After a fixed number of local steps, the Fed-Server aggregates all participated $\boldsymbol{\theta}_{l, i}$ (e.g., via weighted averaging like FedAvg~\cite{mcmahan2017communication}) and broadcasts updated global model $\bar{\boldsymbol{\theta}}_{l}$  to all clients to initiate the next round.

\subsection{Zeroth-Order Gradient Estimator}
Unlike prior methods \cite{mu2025federated, han2021accelerating, nair2025fsl} that rely on full forward and backward passes through the client and its auxiliary network to compute first-order gradients $\nabla\ell(\boldsymbol{\theta}_l;\xi_i)$, we adopt a mini-batch-type ZO gradient estimator with two-point evaluation. Specifically, for function $f_{l,i}$, the two-point type stochastic ZO gradient estimator is defined as:
\begin{equation}\label{Eq:ZerothOrderGradientEstimator}
    \begin{aligned}
        \hat{\nabla} f_{l,i}(\boldsymbol\theta_{l};\xi_i)
        = \frac{1}{B}\sum\nolimits_{j=1}^{B}\frac{d\boldsymbol u}{\mu}[\ell_{l,i}(\boldsymbol\theta_{l} + \mu \boldsymbol u;\xi_{i, j}) - \ell_{l,i}(\boldsymbol\theta_{l};\xi_{i, j}) ],
    \end{aligned}
\end{equation}
where $\boldsymbol u$ is a random vector drawn from either a Gaussian or a Uniform ball distribution, $\mu$ is a positive perturbation step size. 
This estimator approximates the smoothed objective function's gradient. Formally, it can be shown that this estimator is an unbiased estimate of $\nabla f_{l,i}^\mu(\boldsymbol{\theta}_{l})$, where $f_{l,i}^\mu$ is the Gaussian-smoothed surrogate of the original function $f_{l,i}$. 
This estimator approximates the gradient of a smoothed surrogate objective, which is defined below.

\begin{definition}[Gaussian Smoothed Function with Unit-Sphere Normalization]
A function $f:\mathbb{R}^d\to\mathbb{R}$ is said to be spherically smoothed with radius $\mu>0$ if
\begin{equation}
f^\mu(\boldsymbol{x})
= \mathbb{E}_{\boldsymbol{z}\sim\mathcal{N}(0,I_d)}
\!\big[f\!\big(\boldsymbol{x}+\mu\,\frac{\boldsymbol{z}}{\|\boldsymbol{z}\|}\big)\big],
\end{equation}
where $\boldsymbol{u}:=\boldsymbol{z}/\|\boldsymbol{z}\|$ satisfies $\|\boldsymbol{u}\|=1$ almost surely and
$\boldsymbol{u}\sim\mathrm{Unif}(\mathbb{S}^{d-1})$.
\end{definition}

We recall a standard result on Gaussian smoothing \cite{nesterov2017random}: 
if $f$ is $L$-smooth, then $f^\mu$ is continuously differentiable with $L_\mu \le L$, and
\begin{equation}
    \nabla f^\mu(\boldsymbol{x})
= \mathbb{E}_{\boldsymbol{u}}
\!\left[\frac{f(\boldsymbol{x}+\mu \boldsymbol{u})-f(\boldsymbol{x})}{\mu} \boldsymbol{u}\right].
\end{equation}
Consequently, the two-point zeroth-order estimator in (2) is an unbiased estimator of the gradient
of the smoothed objective:
\begin{equation}\label{Eq:UnbiasedEstimator}
\mathbb{E}_{\boldsymbol{u}}\!\big[\hat{\nabla} f_{l,i}(\boldsymbol{\theta}_{l};\xi_i)\big]
= \nabla f_{l,i}^\mu(\boldsymbol{\theta}_{l};\xi_i).
\end{equation}
The bias with respect to $\nabla f_{l,i}$ is solely due to smoothing and is controlled by $\mu$.

\section{Proposed Algorithm: HERON-SFL}
\begin{figure}
  \centering
  \includegraphics[width=0.65\linewidth]{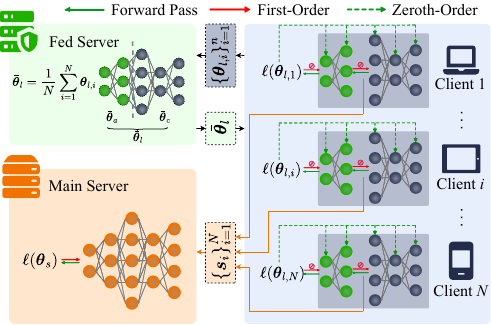}
  \caption{The proposed HERON-SFL algorithm. }
  \label{fig:HERON_SFL_system}
\end{figure}
We present the end-to-end training process of our proposed framework, which operates over a series of communication rounds with synchronized aggregation (high-level illustration depicted in Fig.~\ref{fig:HERON_SFL_system}).
Each round, indexed by $t$, encompasses four key stages: model initialization, local client computation, server-side updates, and local model aggregation in Fed-Server. 
The entire process is formalized as follows:

\textbf{1. Model Initialization.} 
At the start of the $t$-th communication round, the Fed-Server broadcasts the global model parameters  ${\boldsymbol{\theta}}_{c}^{t}$ and ${\boldsymbol{\theta}}_{a}^{t}$ that are resulted from the federated aggregation at the end of the last round. Upon receiving these parameters, each client $i$ initializes its local models for the subsequent update process: $\boldsymbol{\theta}_{l,i}^{t,0}=\{\boldsymbol{\theta}_{c,i}^{t,0},\boldsymbol{\theta}_{a,i}^{t,0}\} = \{{\boldsymbol{\theta}}_{c}^{t},{\boldsymbol{\theta}}_{a}^{t}\}$.

\textbf{2. Local Model Update and Smashed Data Upload.}
The client then proceeds with $h$ local model updates. During this process, the update of the client-side model is {decoupled} from the server-side model by leveraging an auxiliary network. 
Distinct from existing methods, our paradigm employs a {ZO gradient estimator} (defined in Eq.~\eqref{Eq:ZerothOrderGradientEstimator}) to approximate the gradients of a local loss function. This allows the client to perform timely updates without requiring traditional back-propagation from the server. 
After performing $h$ local gradient descent steps, the cumulative update for the client-side models can be concisely written as:
\begin{equation}\label{Eq:ClientLocalUpdate}
    \boldsymbol{\theta}_{l,i}^{t,h} = \boldsymbol{\theta}_{l,i}^{t,0} - \eta_l\sum\nolimits_{m=1}^{h} \hat{\nabla} f_{l,i}(\boldsymbol\theta_{l,i}^{t,m};\xi_i).
\end{equation}
During the local update phase, the client uploads its smashed data to the server every $k$ local steps for the subsequent server-side training phase.

\textbf{3. Server Model Update.}
The server receives the smashed data from each client $i$ and performs model updates sequentially using an SFLV2 \cite{thapa2022splitfed} training scheme. In this setting, each client’s smashed data is processed one-by-one, and standard first-order optimization based on forward and backward propagation is used to estimate gradients and update the server-side model parameters $\boldsymbol\theta_s^t$ accordingly:
\begin{equation}\label{Eq:ServerUpdate}
    \boldsymbol{\theta}_s^{t+1} = \boldsymbol{\theta}_s^t - \eta_s \sum\nolimits_{i=1}^N \frac{1}{\vert \mathcal D_i\vert}\sum\nolimits_{\xi_i\in\mathcal D_i} \nabla \ell(\boldsymbol\theta_s^t; {\boldsymbol\theta_{c,i}^t}(\xi_i)),
\end{equation} 
where $\nabla_{\boldsymbol\theta_s} \mathit l(\boldsymbol\theta_s^t; {\boldsymbol\theta_{c,{i}}^t}(\xi_i))$ is the real gradient of the server-side loss function using back propagation.

\textbf{4. Model Aggregation in Fed-Server.}
Upon completion of the $h$ local updates, each client transmits its updated local parameters $\boldsymbol\theta_{l,i}^{t,h}$ to the Fed-Server for aggregation. The Fed-Server averages these parameters across all $N$ clients to compute the global model combined by client-side and auxiliary models for the next round:
\begin{equation}\label{Eq:Aggregation}
     \boldsymbol{\theta}_{l}^{t+1}=\bar{\boldsymbol{\theta}}_l^t = \frac{1}{N} \sum\nolimits_{i=1}^{N} \boldsymbol{\theta}_{l,i}^{t,h} 
\end{equation}
The server-side model, $\boldsymbol\theta_s^{t+1}$, which was updated sequentially during the round, is already finalized and requires no aggregation. Finally, the new global model $\boldsymbol\theta_g^{t+1} = \{\boldsymbol\theta_{c}^{t+1}, \boldsymbol\theta_s^{t+1}\}$ is assembled and prepared for distribution in the subsequent communication round.

In essence, HERON-SFL uses a client-side ZO gradient estimator to eliminate backpropagation/activation caching while keeping the same smashed-activation upload schedule as standard decoupled SFL, and thus incurs no additional communication overhead beyond the auxiliary-network-based SFL protocol(e.g., CSE-FSL\cite{mu2025federated} and FSL-SAGE\cite{nair2025fsl}).

\section{Theoretical Analysis}

\subsection{Convergence Analysis}

In this section, we provide a formal convergence analysis to establish the theoretical guarantees for the proposed FSL-HERON framework. The theoretical framework is built upon the following standard assumptions, which are widely adopted in the analysis of distributed optimization algorithms \cite{karimireddy2020scaffold, reddi2020adaptive, mu2025federated, fang2022communication}.

\begin{assumption}[\textbf{L-smoothness}]\label{assumption:smoothness}
\textit{The loss functions of clients and server are $L$-smooth. Mathematically, for any $\boldsymbol{x}\in\mathbb{R}^d$ and $\boldsymbol{y}\in\mathbb{R}^d$, the following holds:
\begin{equation}
\begin{aligned}
    \|\nabla f(\boldsymbol{x}) - \nabla f(\boldsymbol{y})\| \leq L \|\boldsymbol{x} - \boldsymbol{y}\|, \quad f(\boldsymbol{y}) \leq f(\boldsymbol{x}) + \nabla f(\boldsymbol{x})^T (\boldsymbol{y} - \boldsymbol{x}) + \frac{L}{2} \|\boldsymbol{y} - \boldsymbol{x}\|^2,
\end{aligned}
\end{equation}
where $f$ is the loss function, and $L$ is the Lipschitz constant. }
\end{assumption}
\begin{assumption}[\textbf{Bounded gradients}]\label{assumption:bounded_gradients}
\textit{The gradients of the local loss function $\ell_i(\boldsymbol\theta_c, \boldsymbol\theta_s)$ are bounded, i.e., there exists a constant $G$ such that: 
\begin{equation}
    \|\nabla_{\boldsymbol\theta_c} \ell_i(\boldsymbol\theta_c)\|^2 \leq G_c^2,
    \|\nabla_{\boldsymbol\theta_s} \ell_i(\boldsymbol\theta_s)\|^2 \leq G_s^2.
\end{equation}}
\end{assumption}
\begin{assumption}[\textbf{Bounded variance for ZO estimator}]\label{assumption:bounded_variance}
\textit{The variance of the zeroth-order gradient estimator is bounded, i.e., there exists a constant $\sigma^2$ such that:
\begin{equation}
    \mathbb{E}[\|\hat{\boldsymbol{g}}_{c,i}^{t,m} - \nabla_{\boldsymbol\theta_c} f_i(\boldsymbol\theta_c, \boldsymbol\theta_s)\|^2] \leq \sigma^2.    
\end{equation}}
\end{assumption}

\begin{assumption}[\textbf{Uniformly bounded drift of client sub-model}]\label{assumption:distribution_drift}\textit{
    For each client \(i\) at global round \(t\), let $z_{c,i}^t = g_{x_{c,i}^t,h}(z)$ be the output of the \(i\)-th client-side model (with input determined by \(x_{c,i}^t\) and \(\mathcal{D}_i\)), and denote by $P_{c,i}^t(z)$ its output distribution. Let \(P^\ast_{c,i}(z)\) be the reference (time-invariant) output distribution of the \(i\)-th client-side model evaluated at \(x_c^\ast\) and \(\mathcal{D}_i\). Define the distribution distance
\begin{equation}
    d_{c,i}^t := \int_{\mathcal{Z}} \big|\,P_{c,i}^t(z) - P^\ast_{c,i}(z)\,\big|\,dz,
\end{equation}
i.e. the \(L_1\) (total-variation) distance between \(P_{c,i}^t\) and \(P^\ast_{c,i}\).
We assume that the aggregate drift across clients is uniformly bounded as follows: there exists a finite $\delta$ such that
\begin{equation}
    \frac1T\sum\nolimits_{t=1}^T\sum\nolimits_{i=1}^N d_{c,i}^t \le \delta. 
\end{equation}}
\end{assumption}

\begin{remark}
    Together, the Assumptions above ensure a well-behaved optimization environment.  Assumption~\ref{assumption:smoothness} guarantees Lipschitz-continuous gradients and provides the usual quadratic upper bound used in descent arguments; Assumption~\ref{assumption:bounded_gradients} prevents arbitrarily large client/server updates and thus promotes numerical stability; and Assumption~\ref{assumption:bounded_variance} limits the stochastic error between the estimator and the true gradient.
    Assumption~\ref{assumption:distribution_drift} is tailored to the auxiliary-network-assisted FSL setting, as also adopted in \cite{mu2025federated} and motivated by centralized synthetic-gradient frameworks~\cite{belilovsky2020decoupled}. 
    This condition is essential for guaranteeing the stability and convergence of the SFL process under local gradient updates.
\end{remark}

\begin{theorem}[\textbf{Convergence rate of HERON-SFL in the i.i.d. setting}]\label{theorem:convergence_rate_scratch}
Under {Assumptions}~\ref{assumption:smoothness}--\ref{assumption:distribution_drift}, 
if the client learning rate satisfies $\eta_c \leq \{\frac{1}{3Lh}, \frac{2}{NLh^2}, \frac{N}{72L}\}$, 
and is chosen as $\eta_c=\mathcal{O}(\sqrt{{({NB)}/{(dhT)}}})$ while the server learning rate is set to $\eta_s = \mathcal{O}(\sqrt{{{(hB)}/{(dNT)}}})$, and perturbation step size is set to $\mu=\mathcal{O}(1/(dhNBT)^{{1}/{4}})$,
the convergence rate of the HERON-SFl algorithm satisfies:
    \begin{equation}
    \begin{aligned}
        \min_{t\in[T]}\ \mathbb{E}\big[\Vert\nabla f(\boldsymbol{\theta}^t_\text{g})\Vert^2\big]
        \leq
        \mathcal{O}\big(\sqrt{\frac{d}{hNBT}}\big) + \mathcal{O}\big(\sqrt{\frac{1}{dhNBT}}\big).
    \end{aligned}
    \end{equation}
\end{theorem}
\begin{remark}
The convergence bound on the expected gradient norm indicates that the algorithm can achieve a favorable trade-off between the model complexity (characterized by the dimensionality $d$) and the training batchsize (captured by $B$) over the training horizon $T$. 
The bound is dominated by $\mathcal{O}\!(\sqrt{d/(hNBT)})$ (the second term is smaller by $1/\sqrt{d}$). 
{Equivalently, in terms of the communication complexity, achieving an $\varepsilon$-stationary point requires a total of $T = \mathcal{O}\!\big(d/(hNB\varepsilon^2)\big)$ communication rounds. This implies that the algorithm achieves a \textit{linear speedup} with respect to the number of clients $N$ and the local batch size $B$ (i.e., the required rounds scale inversely with the product $NB$).}
Furthermore, increasing the number of local steps $h$ reduces the required communication rounds by a factor of $h$, effectively trading fewer communication rounds for more intensive local computation.
The dependence on model size is $\sqrt{d}$ (or $d$ in sample complexity), which is the drawback of ZO optimization: convergence degrades with increasing dimensionality. Below, we show that the dependency on $d$ can be reduced under structural assumptions on an effective dimension. 
\end{remark}

\begin{assumption}
[\textbf{Low $\kappa$-Effective Rank}\cite{malladi2023fine,he2016deep,li2024achieving}
]\label{assumption:low_rank}
\textit{Let $G_t \triangleq \max_{i, \xi_i \subset \mathcal{D}_i} \|\nabla_{\boldsymbol\theta_l} l_l(\theta_{l,i}^t; \xi_i)\|$. There exists a Hessian matrix $H_l(\theta_{l,i}^t) \preceq L \cdot I_{d_l}$ such that:
\begin{itemize}[leftmargin=10pt]
    \item For all $\boldsymbol\theta_l$ such that $\|\boldsymbol\theta_l - \theta_{l,i}^t\| \le 2\eta_c d_l G_t$, we have $\nabla^2 l_l(\boldsymbol\theta_l) \preceq H_l(\boldsymbol\theta_{l,i}^t)$.
    \item The effective rank of $H_l(\boldsymbol\theta_{l,i}^t)$, i.e., $\frac{\text{tr}(H_l(\boldsymbol\theta_{l,i}^t))}{\|H_l(\boldsymbol\theta_{l,i}^t)\|_2}$, is at most $\kappa$.
\end{itemize}}
\end{assumption}
{{The low effective rank assumption posits that, although the client-side model may be high-dimensional, its local loss landscape is governed by only a few dominant curvature directions. This phenomenon is widely observed in the training dynamics of deep models \cite{malladi2023fine, li2024achieving}; detailed discussion and empirical evidence that supports this assumption in practice} are provided in Appendix~\ref{App:low-rank}.}  

\begin{theorem}[\textbf{Convergence rate of HERON-SFL with Low Effective Rank Assumption}]\label{theorem:low_rank}
Under {Assumptions}~\ref{assumption:smoothness}--\ref{assumption:low_rank},
if the client learning rate satisfies $\eta_c \leq \frac{1}{4L}(1+\frac{d\kappa+d-2}{d+2})$ and $\mu\leq\frac{\sqrt[4]{\kappa}}{\sqrt[4]{NT}\,\sqrt{(d+3)^3}}$, 
and is chosen as $\eta_c=\mathcal{O}(\sqrt{{({NB)}/{(\kappa T)}}})$ while the server learning rate is set to $\eta_s = \mathcal{O}(\sqrt{B/{(\kappa NT)}})$, the convergence rate of the HERON-SFL algorithm satisfies:
\begin{equation}
\begin{aligned}
    \min_{t\in[T]}\mathbb{E}\big[\Vert\nabla f(\boldsymbol{\theta}^t_\text{g})\Vert^2\big]  \leq \mathcal{O}\big(\sqrt{\frac{\kappa}{NBT}}\big) + \mathcal{O}\big(\frac{1}{T}\big) + \frac{2}{\delta}\!\big[\frac{2G_s^2}{N(2N-1)}\Delta+\frac{\mu^2L^2}{2}(d+3)^3\big].
\end{aligned}
\end{equation}
\end{theorem}
\begin{remark}
With the prescribed $\mu$, the smoothing bias term $\propto \mu^2(d+3)^3$ is at most $\mathcal{O}\big(\sqrt{\kappa/(NT)}\big)$ and the drift term vanishes in the i.i.d.\ case ($\Delta=0$). The bound thus simplifies to $\mathcal{O}\big(\sqrt{\kappa/(NBT)}\big)+\mathcal{O}(1/T)$, which is independent of the model dimension $d$, removing the usual $\sqrt{d}$ degradation of ZO methods and matching the $1/\sqrt{T}$ rate of FO SFL  \cite{mu2025federated, nair2025fsl} up to condition number $\kappa$ factors. 
\end{remark}

\begin{table}[h!]
\caption{Client-Side Resource Costs per Local Update.}
\centering
\setlength{\tabcolsep}{5pt}
\label{tab:resource_comparison}
\renewcommand{\arraystretch}{1.5} %
\begin{tabular}{l|c|c|c}
\hline
\textbf{Method} & \textbf{Comms. per Client} & \textbf{Peak Memory} & \textbf{FLOPs} \\
\hline
{SFLV1/V2} & $2pq + 2\vert\boldsymbol{\theta}_c\vert$ & $\mathcal{O}(\vert\boldsymbol{\theta}_c\vert)$ & $3 F_c$ \\\hline
{CSE-FSL} & $pq + 2(\vert\boldsymbol{\theta}_c\vert+\vert\boldsymbol{\theta}_a\vert)$ & $\mathcal{O}(\vert\boldsymbol{\theta}_c\vert+\vert\boldsymbol{\theta}_a\vert)$ & $3 (F_c + F_a)$ \\\hline
{HERON-SFL} & $pq + 2(\vert\boldsymbol{\theta}_c\vert+\vert\boldsymbol{\theta}_a\vert)$ & $\mathcal{O}(1)$ & $n_p(F_c + F_a)$ \\
\hline
\end{tabular}
\end{table}
\subsection{Client-side Resource Cost Analysis}
The following analysis, summarized in Table \ref{tab:resource_comparison}, compares the per-client resource consumption for a single parameter update step on a fixed-size batch of data, assuming all other hyperparameters are kept constant. Let $p$ be the data size of one local batch, $q$ be the size of the smashed layer, and $\vert\boldsymbol{\theta}_c\vert$, $\vert\boldsymbol{\theta}_a\vert$ be the size of the client-side and auxiliary models, respectively.

\subsubsection{Communication Load}
The key communication advantage of decoupled frameworks (CSE-FSL, FSL-SAGE, and HERON-SFL) over traditional SFL (SFLV1/V2) is the removal of server-to-client gradient downloads. In traditional SFL, each batch requires a two-way exchange of intermediate data (captured by $2pq$), whereas decoupled methods only upload smashed data, reducing the per-batch cost to $pq$. This saving comes at the cost of exchanging auxiliary-model parameters, $\vert\boldsymbol{\theta}_a\vert$, which is typically minor compared to transmitting smashed data. Importantly, HERON-SFL’s ZO local updates are fully performed on-device and do not introduce any additional intermediate-data transmissions; thus, \textit{HERON-SFL incurs no extra communication overhead beyond the standard decoupled SFL protocol}.

\subsubsection{Peak Memory}
FO frameworks like SFLV1/V2 and CSE-FSL require caching intermediate activations for backpropagation. This results in a peak memory footprint that scales with the size of the locally trained models, i.e., $\mathcal O(\vert\boldsymbol{\theta}_c\vert)$ and $\mathcal O(\vert\boldsymbol{\theta}_c\vert+\vert\boldsymbol{\theta}_a\vert)$ respectively. This overhead can be an order of magnitude larger than that of inference \cite{griewank2008evaluating}. 
In contrast, ZO-based HERON-SFL eliminates activation caching and reduces peak memory to $\mathcal{O}(1)$, matching inference cost \cite{malladi2023fine}.

\begin{remark}
    Local ZO updates are highly memory-efficient for two primary reasons. First, they eliminate the need for backpropagation, thus avoiding the high cost of caching intermediate activations. Second, the perturbed parameters $\boldsymbol{u}$ generated in the calculation $\hat{\nabla} f_{l,i}(\boldsymbol\theta_{l};\xi_i))$ do not require storing the full perturbation vector; instead, the vector can be procedurally generated from a single random seed and applied in-place, further minimizing memory overhead. 
\end{remark}

\subsubsection{FLOPs}
Assuming a backward pass is twice as computationally expensive as a forward pass ($F$) \cite{chen2016training}, first-order methods incur a cost of approximately $3F_c$ (for SFLV1/V2) or $3(F_c + F_a)$ (for CSE-SFL and FSL-SAGE) per update, where $F_c$ and $F_a$ are the forward pass costs of the client and auxiliary models, respectively. In contrast, HERON-SFL performs ZO updates directly on the client, similar to the approach in MeZO \cite{malladi2023fine}. In practice, a standard two-point ZO estimator is typically sufficient for stable and effective parameter updates, requiring a computational cost of $2(F_c + F_a)$ in HERON-SFL.

\section{Experiments}
\subsection{Experiment Setting}In this section, we conduct experiments on both model training and fine-tuning to show the performance of our proposed HERON-SFL algorithm. 
For comparison, we use the following baseline methods: 
SFLV1/V2~\cite{thapa2022splitfed} or SplitLoRA\footnote{While SFLV1/V2 are designed for the training-from-scratch paradigm, our focus on the distinct task of language fine-tuning led to the development of SplitLoRA, which integrates LoRA with the SFLV2 framework. We omit a comparison with an SFLV1-based approach because its need for multiple server models is computationally prohibitive for large-scale models.}~\cite{lin2024splitlora}, {CSE-FSL} \cite{mu2025federated}, and {FSL-SAGE}~\cite{nair2025fsl}. We conduct the experiments under two complementary training paradigms, implementing all models in PyTorch and running them on NVIDIA RTX A6000 NVL GPU (48 GB):

\textbf{Full Training from Scratch}. We study the convergence of ResNet-18 \cite{he2016deep} under SFL on CIFAR-10 \cite{krizhevsky2009learning} with 5 clients. 
The model is split after the second 2-D BatchNorm layer; the client holds the front part while the server holds the back part. 
An auxiliary head consisting of a single fully connected layer is attached to the cut layer. 
Unless otherwise stated, we adopt the hyperparameters in \cite{thapa2022splitfed}: batch size 256 and Adam optimizers on both sides with a learning rate of $1e{-4}$.

\textbf{Language Model Fine-tuning}. We fine-tune GPT2-Small and GPT2-Medium \cite{radford2021learning} on the E2E dataset \cite{novikova2017e2e} with 3 clients. Unless specified otherwise, for GPT2-Small, the model is split after the third transformer block, with an auxiliary network consisting of one transformer block and the unembedding layer. For GPT2-Medium, the split occurs after the sixth block, with a three-block auxiliary network plus the unembedding layer. As the auxiliary network is not pre-trained, we initialize its parameters by copying the weights from the initial blocks of the server-side model. All components are fine-tuned using Low-Rank Adaptation (LoRA) \cite{hu2022lora}, where only adapters of rank 8 are updated and all other parameters are frozen.

The former setting evaluates whether SFL can train a model \emph{from scratch}, a prerequisite when no reliable checkpoint exists. The latter mirrors the prevailing industrial practice of pre-training a large language model once and then adapting it with parameter- and memory-efficient techniques such as LoRA. 
By examining both regimes, we separately measure the contributions of data-parallel federation, model partitioning, and parameter-efficient adapters, and we show that {HERON-SFL} consistently outperforms strong baselines in both scenarios.

\subsection{Training from Scratch: ResNet18 on CIFAR-10}
\begin{figure}
    \centering
    \includegraphics[width=0.85\linewidth]{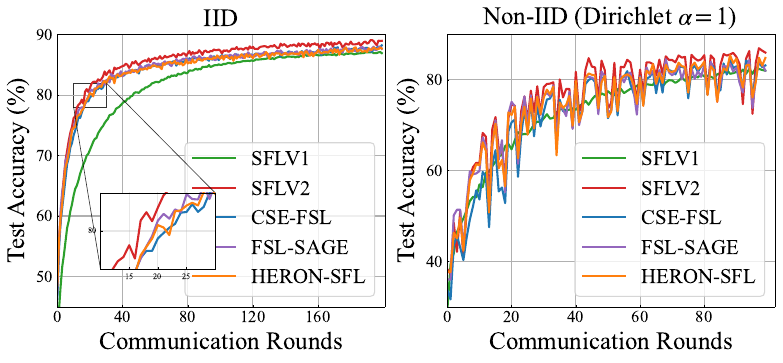}
    \caption{ResNet-18 test accuracy vs. communication rounds on CIFAR-10 for IID (left) and non-IID (right) distributions.}
    \label{fig: testacc-cifar10}
\end{figure}
\subsubsection{Convergence Behavior}
Fig.~\ref{fig: testacc-cifar10} illustrates the test accuracy of each method versus the number of communication rounds. In the IID setting, our proposed HERON-SFL shows convergence behavior nearly identical to other auxiliary-network baselines like CSE-FSL and FSL-SAGE\footnote{We note that FSL-SAGE does not exhibit a significant advantage in our experiments, which we attribute to our design choice of using a minimal auxiliary network purely for decoupling the updates of server and clients. This contrasts with the approach in \cite{nair2025fsl}, where the alignment mechanism of FSL-SAGE is more impactful as the auxiliary model is intentionally designed to be even larger than the client model, thus requiring explicit alignment to ensure consistency with the server's task.}, with all three performing slightly below the top-performing SFLV2. A similar trend is observed in the more challenging non-IID setting, which confirms that our hybrid algorithm achieves convergence comparable to its first-order counterparts.

\begin{table}[!t]
\caption{Client consumptions for ResNet-18 on CIFAR-10.\label{tab:cifar_consumption}}
\centering
\setlength{\tabcolsep}{8pt}
\renewcommand{\arraystretch}{1.35}
\begin{tabular}{c|c|c|c}
\hline
\textbf{Algorithm} & \textbf{Comm. (GB)} & \textbf{Peak FP (MB)} & \textbf{FLOPS (G)}\\
\hline
SFLV1 & 1216.00 & \multirow{2}{*}{\underline{709.93}} & \multirow{2}{*}{\underline{59.51}} \\
\cline{1-2}
SFLV2 & 390.67 &  &  \\
\hline
CSE-FSL & 258.55 & \multirow{2}{*}{726.46} & \multirow{2}{*}{59.85} \\
\cline{1-2}
FSL-SAGE & \underline{244.24} &  &  \\
\hline
HERON-SFL & \textbf{244.19} & \textbf{259.44} & \textbf{39.90} \\
\hline
\end{tabular}
\end{table}

\subsubsection{Communication, Storage, and Computational Costs}
Table~\ref{tab:cifar_consumption} summarizes the client-side resource costs, where the reported communication volume is the cumulative traffic incurred until the test accuracy first reaches 80\%. Under this criterion, HERON-SFL is among the most communication-efficient methods, requiring only 244.19~GB, essentially matching FSL-SAGE (244.24~GB) and improving over the remaining baselines.
The most significant advantages of HERON-SFL are evident in its on-device resource requirements. By eliminating client-side backpropagation, it drastically reduces the {peak memory footprint (Peak FP)} to just 259.44 MB---a reduction of approximately 64\% compared to the almost 710 MB required by SFLV1 and SFLV2. Similarly, the {computational cost (FLOPs)} is lowered to 39.90 G FLOPs, a reduction of over 33\% compared to the \textasciitilde59 G FLOPs of other methods. This substantial decrease in both storage and compute burden confirms that HERON-SFL is highly suitable for deployment in resource-constrained environments.

\begin{figure*}[!t]
\centering
\subfloat[Effect of data heterogeneity.\label{fig:non_iid}]
{\includegraphics[width=0.33\textwidth]{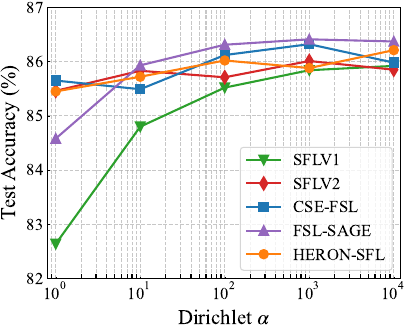}}\hfil
\subfloat[Effect of the number of clients.\label{fig:client_num}]
{\includegraphics[width=0.325\textwidth]{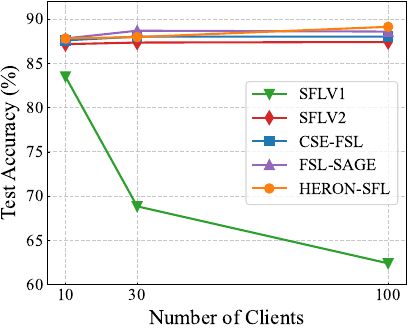}}\hfil
\subfloat[Effect of partial participation.\label{fig:partial}]
{\includegraphics[width=0.325\textwidth]{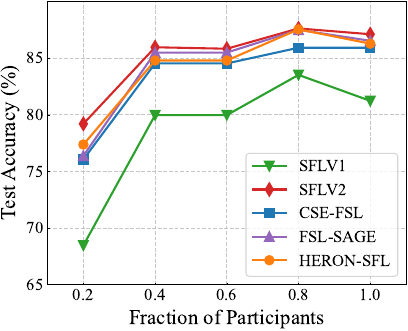}}
\caption{Test accuracy of different SFL algorithms on CIFAR-10 using a ResNet-18 model. 
    (a) Impact of data heterogeneity under varying Dirichlet $\alpha$ values. 
    (b) Client scalability under different total numbers of clients. 
    (c) Performance under different fractions of participating clients per round.}
\label{fig:cifar10_three_ablation}
\end{figure*}

\subsubsection{Effect of Data Heterogeneity (Non-IID)}

The impact of data heterogeneity is evaluated on CIFAR-10 using a ResNet-18 model with ten clients under full participation. As shown in Fig.\ref{fig:non_iid}, varying the Dirichlet concentration parameter creates a broad range of non-IID conditions, yet HERON-SFL maintains accuracy comparable to first-order SFL baselines across all levels of heterogeneity. The zeroth-order updates do not weaken the model’s ability to handle distributional shifts, and the perturbation-induced noise remains well controlled. These results indicate that HERON-SFL preserves the robustness to non-IID client data.

\subsubsection{Effect of the Number of Clients}
Scalability is examined by varying the total number of clients while keeping the dataset (CIFAR-10), model architecture (ResNet-18), and full participation unchanged under an IID configuration. As shown in Fig.\ref{fig:client_num}, HERON-SFL sustains nearly identical accuracy as the federation expands from ten to one hundred clients, demonstrating that HERON-SFL remains stable at larger scales. 

\subsubsection{Effect of Partial Participation}
Training under partial participation is evaluated on CIFAR-10 with ResNet-18 and 10 IID clients. 
As shown in Fig.~\ref{fig:partial}, HERON-SFL remains stable across a wide range of participation ratios, even when only a small fraction of clients contribute per round. Its accuracy closely tracks first-order SFL baselines, suggesting that partial participation does not affect the convergence of the zeroth-order client updates. Overall, HERON-SFL is robust to limited participation, as commonly observed in cross-device FL.

\subsubsection{Effect of Hyperparameters of ZO}
\begin{figure}[t]
\centering
\subfloat{
    \includegraphics[width=0.4\linewidth]{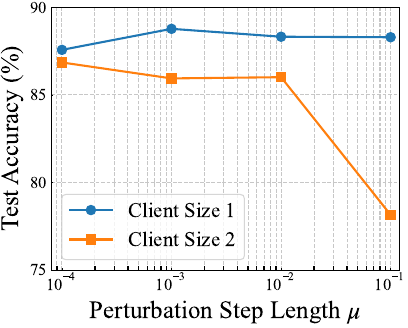}
}\hspace{1em}
\subfloat{
    \includegraphics[width=0.4\linewidth]{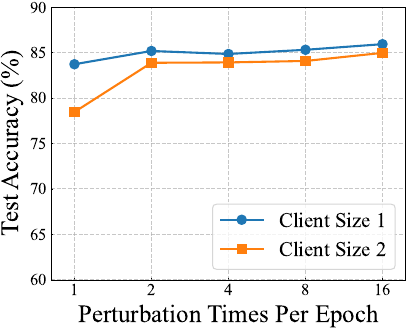}
}
\caption{Ablation study on local ZO training hyperparameters using ResNet-18 on CIFAR-10 under an IID setting with ten clients, with all experiments using the same auxiliary model implemented as a single linear layer. Client Size 1 denotes the first convolutional layer and one residual block on the client, while Client Size 2 places three residual blocks on the client. (left) Test accuracy under different perturbation step lengths $\mu$. (right) Test accuracy under different perturbation counts per epoch.}
\label{fig:ablation_zo}
\end{figure}

As shown in Fig.~\ref{fig:ablation_zo}, HERON-SFL exhibits stable performance across a wide range of zeroth-order hyperparameters, demonstrating robustness to both the perturbation step size and the number of perturbations per epoch. When an appropriate step size $\mu$ is selected, using only two perturbations per epoch is sufficient to ensure reliable convergence, indicating that HERON-SFL does not suffer from the instability often associated with zeroth-order optimization. Across both figures, Client Size 1 consistently achieves higher accuracy than Client Size 2, reflecting the expected increase in optimization difficulty when a larger portion of the model is placed on the client. This mild degradation is acceptable in SFL settings because resource-constrained devices typically hold only small client sub-models, while the majority of parameters remain on the server for first-order training. Overall, the results confirm that HERON-SFL maintains strong accuracy under practical ZO configurations and remains reliable even when client-side capacity varies.

\subsection{Language Model Fine-tuning}
\subsubsection{Training Behavior}
For the task of language model fine-tuning, HERON-SFL demonstrates superior communication efficiency and faster convergence. As illustrated in Fig.\ref{fig:e2e_ppl}, its validation perplexity decreases more rapidly than the baselines for both GPT2-Small and GPT2-Medium. Notably, for GPT2-Small, HERON-SFL converges faster and achieves a final perplexity that is competitive with SplitLoRA while outperforming both CSE-FSL and FSL-SAGE. While all methods reach a similar performance on GPT2-Medium, HERON-SFL does so with significantly less communication costs, and even slightly surpasses CSE-FSL and FSL-SAGE on GPT2-Small. This mild performance gain is consistent with recent findings in ZO-based LM fine-tuning, where the update landscape exhibits strong low-rank structure, making zeroth-order steps exceptionally effective. Similar behavior is reported in MeZO \cite{malladi2023fine}, which shows that ZO fine-tuning can match or even surpass first-order methods under comparable settings.
\begin{figure}
    \centering
    \includegraphics[width=0.85\linewidth]{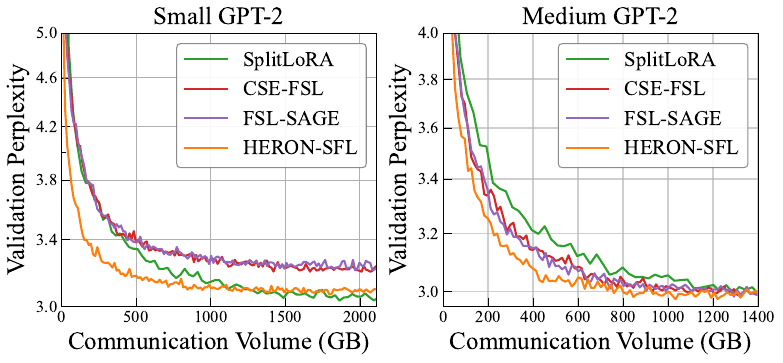}
    \caption{GPT2 perplexity curves vs. Communication Volume on E2E for small (left) and medium (right) models.}
    \label{fig:e2e_ppl}
\end{figure}

\begin{table}[!t]
\caption{Client consumptions for GPT2-Medium on E2E dataset.\label{tab:e2e_consumption}}
\centering
\setlength{\tabcolsep}{8pt}
\renewcommand{\arraystretch}{1.35}
\begin{tabular}{c|c|c}
\hline
\textbf{Algorithm} & \textbf{Peak FP (GB)} & \textbf{FLOPS (T)}\\
\hline
SplitLora & \underline{4.59} & \underline{5.68} \\
\hline
CSE-FSL & \multirow{2}{*}{9.09} & \multirow{2}{*}{9.48} \\
\cline{1-1}
FSL-SAGE &  &  \\
\hline
HERON-SFL & \textbf{4.03} & \textbf{5.26} \\
\hline
\end{tabular}
\end{table}

\subsubsection{Storage and Computational Costs}
Echoing the resource efficiency observed in the ResNet experiments, HERON-SFL substantially lowers the on-device computational and memory burden for clients. Table \ref{tab:e2e_consumption} provides a clear comparison of the resource consumption per local update. HERON-SFL requires a peak memory (Peak FP) of only 4.03 GB, which is less than half that of CSE-FSL (9.09 GB) and also more efficient than the SplitLoRA baseline (4.59 GB). 
The reduction in computational cost is even more pronounced, with HERON-SFL needing only 5.26 TFLOPS, a decrease of approximately 44\% compared to CSE-FSL and FSL-SAGE. This reduction in both memory footprint and floating-point operations confirms that by eliminating client-side backpropagation, our method significantly lowers the hardware barrier, making it feasible to fine-tune language models on resource-constrained devices.

\subsubsection{Ablation study of local model complexity}`
We investigate the impact of local model complexity on the GPT2-medium fine-tuning task. 
\begin{figure}
  \centering
  \includegraphics[width=0.5\linewidth]{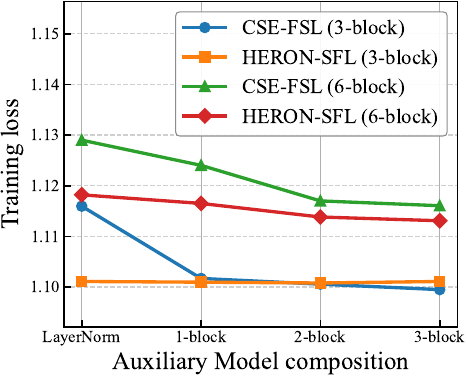}
  \caption{Effect of aux-model complexity on the GPT2-medium fine-tuning task.}
  \label{fig:ablation_aux}
\end{figure}
We ablate auxiliary-model complexity on the GPT2-medium fine-tuning task under two client partitions, where the client-side model contains either the first 3 or 6 transformer blocks. For each setting, we vary the auxiliary network from a minimal design (LayerNorm and unembedding layers only) to larger variants with 1–3 transformer blocks. Fig.\ref{fig:ablation_aux} reports the final training loss after a fixed number of communication rounds. 

HERON-SFL is largely insensitive to the auxiliary model’s capacity: in both settings, it achieves strong final loss even with the minimal auxiliary network. In contrast, the first-order baseline CSE-FSL benefits substantially from a stronger auxiliary model, with performance improving as the auxiliary network grows. These results suggest that ZO-based client updates do not require a resource-intensive auxiliary network, whereas FO baselines rely on auxiliary capacity to reach their full potential. Overall, HERON-SFL delivers strong convergence while reducing peak client memory to inference-level by eliminating client-side backpropagation, providing a better performance–cost trade-off than FO baselines without introducing additional communication overhead.

\section{Conclusion}
In this work, we have proposed HERON-SFL, a novel hybrid ZO-FO framework that addresses the critical computation and memory limitations on edge devices within SFL. By performing zeroth-order optimization on the client side, HERON-SFL eliminates the need for backpropagation and activation caching during local updates, thereby significantly reducing on-device computational and memory requirements, while operating under the same communication budget as existing auxiliary-network-based SFL frameworks.
Empirical and theoretical analysis demonstrated that the framework not only achieves a theoretical convergence rate of $\mathcal O(1/\sqrt{T})$ independent of model dimensionality under the low effective rank assumption, but also empirically matches the accuracy of SFL benchmarks on diverse tasks while substantially reducing client-side resource costs.
Future work may explore non-differentiable objectives—for example, directly optimizing evaluation metrics or incorporating human feedback \cite{ouyang2022training}, which align well with the gradient-free nature of client-side updates. 
Another promising direction is to strengthen privacy guarantees: HERON-SFL enjoys the same privacy properties as standard SL/SFL due to its cut-layer design, but the privacy protection of intermediate activations in SFL can be further strengthened~\cite{niu2024all}.

\appendices
\section{Proofs of Theorems}\label{appendix:theoretical_proof}

\subsection{Notation and Problem Setup Recap}
We denote the global model parameters at round $t$ as $\boldsymbol\theta_{g}^t = \{\boldsymbol\theta_c^t, \boldsymbol\theta_s^t\}$, and the local model parameters at client $\mathcal{C}_i$ as $\boldsymbol\theta_{l,i}^t = \{\boldsymbol\theta_{c,i}^t, \boldsymbol\theta_{a,i}^t\}$.
According to the local update (\ref{Eq:ClientLocalUpdate}) ($\hat{\boldsymbol{g}}_{c,i}^{t,m}=\hat{\nabla} f_{c,i}^{t,m}(\boldsymbol\theta_{c};\xi_i))$) at clients and the aggregation at Fed Server, each communication round in (\ref{Eq:Aggregation}), we have:
\begin{equation}\label{Eq:ClientUpdateGlobal}
    \boldsymbol\theta_\text{c}^{t+1} - \boldsymbol\theta_\text{c}^t = \boldsymbol\theta_\text{c}^{t, h} - \boldsymbol\theta_\text{c}^t = -\frac{\eta_c}{N}\sum\nolimits_{i=1}^{N} \sum\nolimits_{m=1}^h \hat{\boldsymbol{g}}_{c,i}^{t,m},
\end{equation}
We decompose the global model's convergence behavior into client-side and server-side contributions. Same as the Proposition 3.4 and 3.5 in~\cite{han2024convergence}, under Assumptions \ref{assumption:smoothness}, we have:
\begin{equation}\label{eq:decompose}
    \begin{aligned}
    &\mathbb{E}_t[f(\boldsymbol\theta_{g}^{t+1})] - f(\boldsymbol\theta_{g}^t) \leq \mathbb{E}_t[\mathcal C] + \mathbb{E}_t[\mathcal S] \\
    \end{aligned}
\end{equation}
where $\mathcal C = \nabla f(\boldsymbol\theta_\text{c}^t)^T (\boldsymbol\theta_\text{c}^{t+1} - \boldsymbol\theta_\text{c}^t) + \frac{L}{2} \|\boldsymbol\theta_\text{c}^{t+1} - \boldsymbol\theta_\text{c}^t\|^2$ and $\mathcal S = \nabla f(\boldsymbol\theta_\text{s}^t)^T (\boldsymbol\theta_\text{s}^{t+1} - \boldsymbol\theta_\text{s}^t) + \frac{L}{2} \|\boldsymbol\theta_\text{s}^{t+1} - \boldsymbol\theta_\text{s}^t\|^2$ denote the contributions from the client-side and server-side models, respectively. $\mathbb{E}_t[\cdot]$ denotes the expectation on all randomness up to round $t$.

\subsection{Proof of Theorem~\ref{theorem:convergence_rate_scratch}}

\subsubsection{Preliminary Lemmas}
To begin the convergence analysis, we start with a few lemmas that will be useful in the subsequent proofs.

\begin{lemma}[\textbf{Gradient and Smoothness for Gaussian Smoothed Functions}~\cite{nesterov2017random}]
\label{lemma:gaussian_smoothing}
Let $f:\mathbb{R}^d\to\mathbb{R}$ be $L$-smooth. Then $f^\mu$ is continuously differentiable with
$L_\mu\le L$, and its gradient admits the representation
\begin{equation}
\nabla f^\mu(\boldsymbol{x})
= \mathbb{E}_{\boldsymbol{u}}
\big[\frac{f(\boldsymbol{x}+\mu \boldsymbol{u})-f(\boldsymbol{x})}{\mu}\,d\boldsymbol{u}\big]. 
\end{equation}

\end{lemma}

\begin{lemma}[\textbf{Bound on the Second Moment of the ZO Estimator}]\label{lemma:zo_second_moment}
    Under Assumptions~\ref{assumption:smoothness}--\ref{assumption:bounded_variance}, the second moment of the zeroth-order gradient estimator $\boldsymbol{\hat{g}}_{c,i}^{t,m}$ is bounded as follows:
    \begin{equation}
    \begin{aligned}
    &\mathbb{E}_{t,m}\big[\|\hat{\boldsymbol{g}}_{c,i}^{t,m} \|^2\big]     \leq \frac{2dG_c^2}{B} + \frac{d^2L^2\mu^2}{2B} + 2\mu^2L^2 + 6\sigma^2_c  + 6\|\nabla f_{c}(\boldsymbol \theta_{c}^{t})\|^2 + 6L^2\mathbb{E}_{t,m-1}\big[\big\|\boldsymbol \theta_{c}^{t} - \boldsymbol \theta_{c,i}^{t,m}\big\|^2\big],
    \end{aligned}
    \end{equation}
\end{lemma}

\begin{proof}
We expand $\mathbb{E}_{t,m}\!\big[\|\hat{\boldsymbol{g}}_{c,i}^{t,m}\|^2\big]$ by conditioning on the randomness at step $(t,m)$, yielding
\begin{equation}
    \begin{aligned}
        &\mathbb{E}_{t,m}\big[\|\hat{\boldsymbol{g}}_{c,i}^{t,m} \|^2\big]
        \overset{(a)}{=} \mathbb{E}_{t,m-1}\big[ \mathbb{E}_t^m\big[\|\hat{\boldsymbol{g}}_{c,i}^{t,m} \|^2\big] \big] \\
        \overset{(b)}{=}& \mathbb{E}_{t,m-1}\big[ \mathrm{Var}_t^m(\hat{\boldsymbol{g}}_{c,i}^{t,m}) + \|\mathbb{E}_t^m[\hat{\boldsymbol{g}}_{c,i}^{t,m}]\|^2 \big] \\
        \overset{(c)}{=}& \mathbb{E}_{t,m-1}\big[ \mathrm{Var}_t^m(\hat{\boldsymbol{g}}_{c,i}^{t,m}) \big] + \mathbb{E}_{t,m-1}\big[ \|\nabla f_{c,i}^\mu(\boldsymbol \theta_{c,i}^{t,m})\|^2 \big]\\
        \overset{(d)}{\leq}&\frac{1}{B} \mathbb{E}_t^m\big[\big\Vert \hat{\boldsymbol{g}}_{c,i}^{t,m}(\boldsymbol\theta_{l};\xi_{i,1})\big\Vert^2\big]+ \mathbb{E}_{t,m-1}\big[\big\| \nabla f_{c,i}^\mu(\boldsymbol \theta_{c,i}^{t,m}) \big\|^2 \big].
    \end{aligned}
\end{equation}
where (a) follows from the tower property $\mathbb{E}_{t,m}[\cdot]=\mathbb{E}_{t,m-1}\big[\mathbb{E}_t^m[\cdot]\big]$, (b) applies $\mathbb{E}[\|\boldsymbol{a}\|^2] = \mathrm{Var}(\boldsymbol{a}) + \|\mathbb{E}[\boldsymbol{a}]\|^2$, (c) follows the outcome of Lemma~\ref{lemma:gaussian_smoothing}, and (d) applies the i.i.d. mini-batch property and $\mathrm{Var}(X) \le \mathbb{E}[\|X\|^2]$.
We now bound the two terms separately. For the first term, we use the bound for two-point estimators (Lemma 4.1 in \cite{gao2018information}) and Assumption~\ref{assumption:bounded_gradients}:
\begin{equation}
\begin{aligned}
    \mathbb{E}_{t,m}\big[\big\Vert \hat{\boldsymbol{g}}_{c,i}^{t,m}(\boldsymbol\theta_{l};\xi_{i,1})\big\Vert^2\big]
    \leq& 2d \mathbb{E}_{t,m}\big[\Vert\nabla \ell_{c,i}(\boldsymbol{\theta}_{c,i}^{t,m};\xi_{i,1})\Vert^2\big] + \frac{1}{2}d^2L^2\mu^2\\
    \leq& 2dG_c^2 + \frac{1}{2}d^2L^2\mu^2.
\end{aligned}
\end{equation}
For the second term, we use the triangle inequality and $\|a+b\|^2 \le 2\|a\|^2 + 2\|b\|^2$:
\begin{equation}
    \begin{aligned}
    \mathbb{E}_{t,m-1}\big[\big\| \nabla f_{c,i}^\mu(\boldsymbol \theta_{c,i}^{t,m}) \big\|^2 \big] 
    \leq& 2\mathbb{E}_{t,m-1}\big[\big\| \nabla f_{c,i}^\mu(\boldsymbol \theta_{c,i}^{t,m}) - \nabla f_{c,i}(\boldsymbol \theta_{c,i}^{t,m}) \big\|^2 \big] + 2\mathbb{E}_{t,m-1}\big[\big\| \nabla f_{c,i}(\boldsymbol \theta_{c,i}^{t,m}) \big\|^2 \big]\\
    \leq& 2\mu^2L^2 + 2\mathbb{E}_{t,m-1}\big[\big\| \nabla f_{c,i}(\boldsymbol \theta_{c,i}^{t,m}) \big\|^2 \big].
    \end{aligned}
\end{equation}
Finally, we bound the remaining term by relating it to the global model state $\boldsymbol\theta_c^t$. Using inequality $||a+b+c||^2 \le 3||a||^2 + 3||b||^2 + 3||c||^2$, we have:
\begin{equation}
    \begin{aligned}
    \mathbb{E}_{t,m-1}\big[\big\| \nabla f_{c,i}(\boldsymbol \theta_{c,i}^{t,m}) \big\|^2 \big] 
    =& \mathbb{E}_{t,m-1}[| (\nabla f_{c,i}(\boldsymbol \theta_{c,i}^{t,m}) - \nabla f_{c,i}(\boldsymbol \theta_{c}^{t})) \\&+ (\nabla f_{c,i}(\boldsymbol \theta_{c}^{t}) - \nabla f_{c}(\boldsymbol \theta_{c}^{t})) + \nabla f_{c}(\boldsymbol \theta_{c}^{t})\|^2]\\
    \leq& 3\mathbb{E}_{t,m-1}\big[\big\| \nabla f_{c,i}(\boldsymbol \theta_{c,i}^{t,m}) - \nabla f_{c,i}(\boldsymbol \theta_{c}^{t})\big\|^2\big] \\&+ 3\|\nabla f_{c,i}(\boldsymbol \theta_{c}^{t}) - \nabla f_{c}(\boldsymbol \theta_{c}^{t})\|^2 + 3\|\nabla f_{c}(\boldsymbol \theta_{c}^{t})\|^2\\
    \leq& 3L^2\mathbb{E}_{t,m-1}\big[\big\|\boldsymbol \theta_{c,i}^{t,m} - \boldsymbol \theta_{c}^{t}\big\|^2\big] + 3\sigma^2_c + 3\|\nabla f_{c}(\boldsymbol \theta_{c}^{t})\|^2,
    \end{aligned}
\end{equation}
where the final inequality follows from Assumptions~\ref{assumption:smoothness} and~\ref{assumption:bounded_variance}. Combining all these bounds yields the result stated in the lemma.
\end{proof}

\begin{lemma}[\textbf{Bound on Client Model Divergence}]\label{lemma:client_bound}
    For $\eta_c\leq \frac{1}{3Lh}$, we have:
    \begin{equation}
        \begin{aligned}
            &\mathbb{E}_t\big[\frac{1}{N}\sum\nolimits_{i=1}^{N} \sum\nolimits_{m=1}^h \|\boldsymbol\theta_{c,i}^{t,m} - \boldsymbol\theta_\text{c}^t\|^2\big] \\
            \leq&  {3h^3\eta_c^2}\Vert\nabla f_c(\boldsymbol{\theta}_c^t)\Vert^2  + \frac{dG_c^2h^3\eta_c^2}{B}+ \frac{d^2L^2\mu^2h^3\eta_c^2}{4B} + \frac{(6\sigma_c^2+2\mu^2L^2)h^3\eta_c^2}{2}.
        \end{aligned}
    \end{equation}
\end{lemma}

\begin{proof}
We define
$s_c^{t,m} \triangleq \frac{1}{N}\sum\nolimits_{i=1}^{N}\, \mathbb{E}_{t,m}\!\big[\big\|\boldsymbol\theta_{c,i}^{t,m}-\boldsymbol\theta_{c}^t\big\|^2\big]$ for simplicity.
For the $\tau$-th local update, unrolling the client recursion gives $\boldsymbol\theta_{c,i}^{t,\tau}-\boldsymbol\theta_{c}^t= -\,\eta_c \sum\nolimits_{m=0}^{\tau-1}\hat{\boldsymbol g}_{c,i}^{t,m}$.

By Cauchy--Schwarz,
\begin{equation}\label{eq:sc_tau_cs}
\begin{aligned}
s_c^{t,\tau}
&= \frac{1}{N}\sum\nolimits_{i=1}^{N}\, \mathbb{E}_{t,\tau}\!\big[\big\|-\eta_c \sum\nolimits_{m=0}^{\tau-1}\hat{\boldsymbol g}_{c,i}^{t,m}\big\|^2\big]\\
&\le\ \tau\,\eta_c^2 \cdot \frac{1}{N}\sum\nolimits_{i=1}^{N}\sum\nolimits_{m=0}^{\tau-1}\mathbb{E}_{t,\tau}\!\big[\big\|\hat{\boldsymbol g}_{c,i}^{t,m}\big\|^2\big] \\
&\stackrel{\text{(tower)}}{=} \tau\,\eta_c^2 \cdot \frac{1}{N}\sum\nolimits_{i=1}^{N}\sum\nolimits_{m=0}^{\tau-1}\mathbb{E}_{t,m}\!\big[\big\|\hat{\boldsymbol g}_{c,i}^{t,m}\big\|^2\big].
\end{aligned}
\end{equation}
We now invoke the second-moment bound (Lemma~\ref{lemma:zo_second_moment}): for every $m$,
\begin{equation}\label{eq:second_moment_bound}
\begin{aligned}
\frac{1}{N}\sum\nolimits_{i=1}^{N}\mathbb{E}_{t,m}\!\big[\big\|\hat{\boldsymbol g}_{c,i}^{t,m}\big\|^2\big]
\le
6L^2\, s_c^{t,m+1}
+
\underbrace{\Big(6\|\nabla f_c(\boldsymbol\theta_c^t)\|^2 + \frac{2dG_c^2}{B} + \frac{d^2L^2\mu^2}{2B} + 6\sigma_c^2 + 2\mu^2L^2\Big)}_{\displaystyle \triangleq \beta},
\end{aligned}
\end{equation}
by definition of $s_c^{t,\cdot}$, the term
$\frac{1}{N}\sum\nolimits_i \mathbb{E}_{t,m}\!\big[\|\boldsymbol\theta_{c}^{t}-\boldsymbol\theta_{c,i}^{t,m+1}\|^2\big]$
is identified with $s_c^{t,m+1}$.
Combining (\ref{eq:sc_tau_cs}) and (\ref{eq:second_moment_bound}) yields, for each $\tau$,
\begin{equation}\label{eq:sc_tau_recursion}
s_c^{t,\tau}
\le
6L^2\,\tau\,\eta_c^2 \sum\nolimits_{m=0}^{\tau-1} s_c^{t,m+1}
+
\tau^2 \eta_c^2\beta.
\end{equation}
By taking the sum over $\tau=1,\ldots,h$, we have
\begin{equation}\label{eq:sc_sum_recursion}
    \begin{aligned} 
        \sum\nolimits_{\tau=1}^h s_c^{t,\tau}
        &\le
        6L^2\,\eta_c^2 \sum\nolimits_{\tau=1}^h \tau \sum\nolimits_{m=0}^{\tau-1} s_c^{t,m+1}+ \eta_c^2\beta \sum\nolimits_{\tau=1}^h \tau^2\\
        &\le 3h^2L^2\,\eta_c^2 \sum\nolimits_{\tau=1}^h s_c^{t,\tau} + \frac{h(h+1)(2h+1)}{6}\,\eta_c^2 \beta \\&\le 3h^2L^2\,\eta_c^2 \sum\nolimits_{\tau=1}^h s_c^{t,\tau} + \frac{h^3\eta_c^2 \beta}{3},
    \end{aligned}
\end{equation}
where we utilized the fact that $\sum\nolimits_{\tau=1}^h \tau \le \frac{h(h+1)}{2} \le \frac{h^2}{2}$ and $\sum\nolimits_{\tau=1}^h \tau^2 = \frac{h(h+1)(2h+1)}{6} \le \frac{h^3}{3}$. By rearranging the terms, we have:
\begin{equation}
    \begin{aligned}
        &(1-3L^2h^2\eta_c^2)\sum\nolimits_{\tau=0}^h s_c^{t,\tau} 
        \le \frac{h^3\eta_c^2}{3}\big(6\|\nabla f_c(\boldsymbol\theta_c^t)\|^2 + \frac{2dG_c^2}{B} + \frac{d^2L^2\mu^2}{2B} + 6\sigma_c^2 + 2\mu^2L^2\big),\\
    \end{aligned}
\end{equation}
When $\eta_c\leq \frac{1}{3Lh}$, we have $1-3L^2h^2\eta_c^2\geq \frac{2}{3}$ and the lemma's proof is complete.
\end{proof}

\begin{lemma}[\textbf{Bound on the Client-Side Contribution}]\label{lemma:client_contribution_bound}
    Under Assumptions~\ref{assumption:smoothness}--\ref{assumption:bounded_variance}, and for a client learning rate $\eta_c$ satisfying the following conditions:
    \begin{equation}\label{eq:eta_c_conditions}
        \eta_c \le \min\big\{\frac{1}{3Lh}, \frac{2}{NLh^2}, \frac{N}{72L}\big\},
    \end{equation}
    the expectation of the client-side contribution, $\mathcal{C} = \nabla f(\boldsymbol{\theta}_{c}^{t})^{T}(\boldsymbol{\theta}_{c}^{t+1}-\boldsymbol{\theta}_{c}^{t})+\frac{L}{2}\|\boldsymbol{\theta}_{c}^{t+1}-\boldsymbol{\theta}_{c}^{t}\|^{2}$, is bounded as:
    \begin{equation}
        \mathbb{E}_{t}[\mathcal{C}] \le -\frac{\eta_{c}h}{4}\|\nabla f_{c}(\boldsymbol{\theta}_{c}^{t})\|^2 + \Phi_{c}(\eta_{c}),
    \end{equation}
    where $\Phi_{c}(\eta_{c})$ is an error term defined as $\Phi_{c}(\eta_{c}) := \eta_{c}^{2}\big(\frac{6hLdG_{c}^{2}}{N|\xi_{i}|}+\frac{18hL\sigma_{c}^{2}}{N}\big) + \eta_{c}\big(\frac{d^{2}L^{2}h\mu^{2}}{48|\xi_{i}|}+\frac{13hL^{2}\mu^{2}}{12}\big)$.
\end{lemma}

\begin{proof}
Based on the client-side update rule and taking expectation over all randomness up to round $t$, we expand $\mathbb{E}_{t}[\mathcal{C}]$ into two terms:
\begin{equation}\label{eq:client-side}
\begin{aligned}
    \mathbb{E}_{t}[\mathcal{C}] = {-\frac{\eta_c}{N} \big\langle\nabla f(\boldsymbol{\theta}_{c}^{t}), \mathbb{E}_{t}\big[\sum\nolimits_{i=1}^{N}\sum\nolimits_{m=1}^{h}\hat{\boldsymbol{g}}_{c,i}^{t,m}\big]\big\rangle} + {\frac{\eta_c^2 L}{2N^2} \mathbb{E}_{t}\big[\big\|\sum\nolimits_{i=1}^{N}\sum\nolimits_{m=1}^{h}\hat{\boldsymbol{g}}_{c,i}^{t,m}\big\|^{2}\big]}.
\end{aligned}
\end{equation}
We set the first term and the second term on the right hand side of (\ref{eq:client-side}) as $\mathcal{C}_1$ and $\mathcal{C}_2$ respectively.
Using the identity $2\langle a,b\rangle = \|a\|^2+\|b\|^2-\|a-b\|^2$, we rewrite $\mathcal{C}_{1}$ as:
\begin{equation}
\begin{aligned}
    \mathcal{C}_{1} = -\frac{\eta_{c}h}{2}\|\nabla f(\boldsymbol{\theta}_{c}^{t})\|^{2}-\frac{\eta_{c}h}{2}\mathbb{E}_{t}\big[\big\|\frac{1}{Nh}\sum\nolimits_{i=1}^{N}\sum\nolimits_{m=1}^{h}\hat{\boldsymbol{g}}_{c,i}^{t,m}\big\|^{2}\big] +\frac{\eta_{c}h}{2}\mathcal{C}_{1,1},
\end{aligned}
\end{equation}
where $\mathcal{C}_{1,1} \triangleq \mathbb{E}_{t}[\|\frac{1}{Nh}\sum\nolimits_{i,m}(\hat{\boldsymbol{g}}_{c,i}^{t,m}-\nabla f(\boldsymbol{\theta}_{c}^{t}))\|^{2}]$. We bound $\mathcal{C}_{1,1}$ using Jensen's inequality, the triangle inequality, and Assumptions 1 and 3:
\begin{equation}
    \begin{aligned}
    \mathcal{C}_{1,1} \le& \frac{1}{Nh}\mathbb{E}_{t}\big[\sum\nolimits_{i=1}^{N}\sum\nolimits_{m=1}^{h}\|\hat{\boldsymbol{g}}_{c,i}^{t,m}-\nabla f(\boldsymbol{\theta}_{c}^{t})\|^{2}\big] \\
    \le& \frac{2}{Nh}\mathbb{E}_{t}\big[\sum\nolimits_{i=1}^{N}\sum\nolimits_{m=1}^{h}\|\hat{\boldsymbol{g}}_{c,i}^{t,m}-\nabla f(\boldsymbol{\theta}_{c,i}^{t,m})\|^{2}\big] + \frac{2}{Nh}\mathbb{E}_{t}\big[\sum\nolimits_{i=1}^{N}\sum\nolimits_{m=1}^{h}\|\nabla f(\boldsymbol{\theta}_{c,i}^{t,m})-\nabla f(\boldsymbol{\theta}_{c}^{t})\|^{2}\big] \\
    \le& 2\sigma^2 + \frac{2L^2}{Nh}\mathbb{E}_{t}\big[\sum\nolimits_{i=1}^{N}\sum\nolimits_{m=1}^{h}\|\boldsymbol{\theta}_{c,i}^{t,m}-\boldsymbol{\theta}_{c}^{t}\|^{2}\big].
    \end{aligned}
\end{equation}
Substituting this back provides a bound on $\mathcal{C}_1$. For $\mathcal C_2$, according to Cauchy-Schwartz inequality, we have:
\begin{equation}
    \begin{aligned} 
    \mathcal C_2    
    \leq& \eta_c^2L {\mathbb{E}_t\big[\big\|\frac{1}{N}\sum\nolimits_{i=1}^{N} \sum\nolimits_{m=1}^h (\hat{\boldsymbol{g}}_{c,i}^{t,m} - \nabla f_{c,i}^\mu(\boldsymbol\theta_{c,i}^{t,m})) \big\|^2\big] }
    \\&+ {\eta_c^2L} {\mathbb{E }_t\big[\big\|\frac{1}{N}\sum\nolimits_{i=1}^{N} \sum\nolimits_{m=1}^h \nabla f_{c,i}^\mu(\boldsymbol\theta_{c,i}^{t,m})\big\|^2\big]},\\
    \end{aligned}
\end{equation}
where we set the first term on the right hand side as ${\mathcal{C}_{2,1}}$.
Since the gradient estimation errors are independent across clients and have zero mean, the expectations of the cross-terms vanish. Thus, $\mathcal{C}_{2,1}$ simplifies to:
\begin{equation}
    \mathcal{C}_{2,1} = \frac{1}{N^2}\sum\nolimits_{i=1}^{N}\mathbb{E}_t\big[\big\|\sum_{m=1}^{h} \big(\hat{\boldsymbol{g}}_{c,i}^{t,m} - \nabla f_{c,i}^\mu(\boldsymbol\theta_{c,i}^{t,m})\big) \big\|^2\big],
\end{equation}
and according to (\ref{Eq:UnbiasedEstimator}) and Lemma 2 in \cite{wang2021novel}, we have:
\begin{equation}
    \begin{aligned}
    \mathcal{C}_{2,1} =& \frac{1}{N^2}\sum\nolimits_{i=1}^{N} \sum\nolimits_{m=1}^{h} \mathbb{E}_{t,m}\big[\big\| \hat{\boldsymbol{g}}_{c,i}^{t,m} - \nabla f_{c,i}^\mu(\boldsymbol\theta_{c,i}^{t,m}) \big\|^2\big]\\
    \overset{(a)}{\leq}& \frac{1}{N^2}\sum\nolimits_{i=1}^{N} \sum\nolimits_{m=1}^{h} \mathbb{E}_{t,m}\big[\|\hat{\boldsymbol{g}}_{c,i}^{t,m} \|^2\big],
    \end{aligned}
\end{equation}
where $(a)$ holds because $\mathbb{E}[\|\boldsymbol{a}-\mathbb{E}[\boldsymbol{a}]\|^2]\leq\mathbb{E}[\|\boldsymbol{a}\|^2]$. Now by applying the second-moment bound from Lemma~\ref{lemma:zo_second_moment}, and substituting these result back, we have:
\begin{equation}
    \begin{aligned} 
    \mathcal C_2 \leq& \eta_c^2L \mathcal{C}_{2,1} + {\eta_c^2L} {\mathbb{E }_t\big[\big\|\frac{1}{N}\sum\nolimits_{i=1}^{N} \sum\nolimits_{m=1}^h \nabla f_{c,i}^\mu(\boldsymbol\theta_{c,i}^{t,m})\big\|^2\big]},\\
    \end{aligned}
\end{equation}
where $\mathcal{C}_{2,1}$ is bounded as follows:
\begin{equation}
    \begin{aligned}
    \mathcal{C}_{2,1} \leq& \frac{6L^2}{N^2}\sum\nolimits_{i=1}^N\sum\nolimits_{m=1}^h\mathbb{E}_{t, m-1}\big[\Vert \boldsymbol \theta_c^t-\boldsymbol\theta_{c,i}^{t,m}\Vert^2\big] + \frac{6h}{N}\Vert\nabla f_c(\boldsymbol{\theta}_c^t)\Vert^2 + \frac{2dG_c^2h}{NB} + \frac{d^2L^2\mu^2h}{2NB} \\&+ \frac{(6\sigma_c^2+2\mu^2L^2)h}{N}\\
    \leq& \frac{6L^2}{N}\mathbb{E}_t\big[\frac{1}{N}\sum\nolimits_{i=1}^N\sum\nolimits_{m=1}^h\Vert \boldsymbol \theta_c^t-\boldsymbol\theta_{c,i}^{t,m}\Vert^2\big]  + \frac{6h}{N}\Vert\nabla f_c(\boldsymbol{\theta}_c^t)\Vert^2+ \frac{2dG_c^2h}{NB} + \frac{d^2L^2\mu^2h}{2NB} \\&+ \frac{(6\sigma_c^2+2\mu^2L^2)h}{N}.
    \end{aligned}   
\end{equation}

Then, by combining the bounds of $\mathcal \mathcal C^{\prime}_1$ and $\mathcal C_2$, we have:
\begin{equation}
    \begin{aligned}
    &\mathbb{E}_t[\mathcal{C}]\\ 
    \leq&
        (\frac{6\eta_c^2Lh}{N}-\frac{\eta_c h}{2})\Vert\nabla f_c(\boldsymbol{\theta}_c^t)\Vert^2+ \frac{(6\sigma_c^2L+2\mu^2L^3)\eta_c^2h}{N}+ \eta_c hL^2\mu^2 + \frac{2\eta_c^2LdG_c^2h}{NB} + \frac{\eta_c^2d^2L^3\mu^2h}{2NB}
        \\&+(\eta^2L-\frac{\eta_c }{2h}){\mathbb{E }_t\big[\big\|\frac{1}{N}\sum\nolimits_{i=1}^{N} \sum\nolimits_{m=1}^h \nabla f_{c,i}^\mu(\boldsymbol\theta_{c,i}^{t,m})\big\|^2\big]} \\
        & + (\eta_c L^2 + \frac{6\eta_c^2L^3}{N})\mathbb{E}_t\big[\frac{1}{N}\sum\nolimits_{i=1}^{N} \sum\nolimits_{m=1}^h \|\boldsymbol\theta_{c,i}^{t,m} - \boldsymbol\theta_\text{c}^t\|^2\big]  \\
    \overset{(a)}{\leq}&
        (\frac{6\eta_c^2Lh}{N}-\frac{\eta_c h}{2})\Vert\nabla f_c(\boldsymbol{\theta}_c^t)\Vert^2 + \frac{(6\sigma_c^2L+2\mu^2L^3)\eta_c^2h}{N}+ \eta_c hL^2\mu^2 + \frac{2\eta_c^2LdG_c^2h}{NB} + \frac{\eta_c^2d^2L^3\mu^2h}{2NB}
        \\&+ (\eta_c L^2 + \frac{6\eta_c^2L^3}{N^2})\mathbb{E}_t\big[\frac{1}{N}\sum\nolimits_{i=1}^{N} \sum\nolimits_{m=1}^h \|\boldsymbol\theta_{c,i}^{t,m} - \boldsymbol\theta_\text{c}^t\|^2\big].
    \end{aligned}
\end{equation}
where $(a)$ holds if and only if $\eta_c\leq \frac{1}{2hL}$, which means term $(\eta^2L-\frac{\eta_c }{2h}){\mathbb{E }_t[\|\frac{1}{N}\sum\nolimits_{i=1}^{N} \sum\nolimits_{m=1}^h \nabla f_{c,i}^\mu(\boldsymbol\theta_{c,i}^{t,m})\|^2]}$ is non-positive.

Finally, we substitute the bound on the client model divergence from Lemma~\ref{lemma:client_bound} into the expression for $\mathbb{E}_t[\mathcal{C}]$:
\begin{equation}
    \begin{aligned}
    \mathbb{E}_t[\mathcal{C}]
    \leq& \big(\frac{6\eta_c^2Lh}{N}-\frac{\eta_c h}{2}\big)\Vert\nabla f_c(\boldsymbol{\theta}_c^t)\Vert^2+ \eta_c hL^2\mu^2 + \frac{2\eta_c^2LdG_c^2h}{NB} + \frac{\eta_c^2d^2L^3\mu^2h}{2NB} + \frac{(6\sigma_c^2L+2\mu^2L^3)\eta_c^2h}{N} \\
    &+(\eta_c L^2 + \frac{6\eta_c^2L^3}{N}) \big(3h^3\eta_c^2\Vert\nabla f_c(\boldsymbol{\theta}_c^t)\Vert^2 + \frac{dG_c^2h^3\eta_c^2}{B} + \frac{d^2L^2\mu^2h^3\eta_c^2}{4B} + \frac{(6\sigma_c^2+2\mu^2L^2)h^3\eta_c^2}{2}\big).
    \end{aligned}
\end{equation}
To simplify this complex expression, we collect the coefficients for the dominant term, $\|\nabla f_c(\boldsymbol{\theta}_c^t)\|^2$, and the remaining bias terms. We set $\alpha \triangleq \eta_ch^3L^2 + \frac{6\eta_c^2h^3L^3}{N}$ to group terms originating from the client drift bound, the bound on $\mathbb{E}_t[\mathcal{C}]$ can be rewritten as:
\begin{equation}
    \begin{aligned}
    \mathbb{E}_t[\mathcal{C}] \le&  \big( \big(\frac{6Lh}{N} + 3\alpha \big)\eta_c^2 -\frac{\eta_c h}{2}\big)\Vert\nabla f_c(\boldsymbol{\theta}_c^t)\Vert^2+ \eta_c hL^2\mu^2 + \alpha\big(\frac{dG_c^2\eta_c^2}{B} + \frac{d^2L^2\mu^2\eta_c^2}{4B} + \frac{(6\sigma_c^2+2\mu^2L^2)\eta_c^2}{2}\big) \\
    & + \frac{2\eta_c^2LdG_c^2h}{NB} + \frac{\eta_c^2d^2L^3\mu^2h}{2NB} + \frac{(6\sigma_c^2L+2\mu^2L^3)\eta_c^2h}{N}.
    \end{aligned}
\end{equation}
Under sufficiently small learning rate $\eta_c$, the negative term $-\frac{\eta_c h}{2}\|\nabla f_c(\boldsymbol{\theta}_c^t)\|^2$ will dominate the other terms multiplying the squared gradient norm. Specifically, by setting conditions on $\eta_c$ such that $\big(\frac{6Lh}{N} + 3\alpha \big)\eta_c^2 \le \frac{\eta_c h}{4},(\text{e.g., satisfied if } \eta_c \le \mathcal{O}(\frac{N}{Lh^2}))$, 
we can simplify the bound on the gradient term to $-\frac{\eta_c h}{4}\|\nabla f_c(\boldsymbol{\theta}_c^t)\|^2$. After collecting all remaining bias and variance terms, we arrive at the final simplified bound:
\begin{equation}
\begin{aligned}
    \mathbb{E}_t[\mathcal{C}] \leq& -\frac{\eta_c h}{4}\Vert\nabla f_c(\boldsymbol{\theta}_c^t)\Vert^2 + \eta_c^2\big(\frac{6hLdG_c^2}{NB} + \frac{18hL\sigma^2_c}{N}\big) + \eta_c\big(\frac{d^2L^2h\mu^2}{48B} + \frac{13hL^2\mu^2}{12}\big),
\end{aligned}
\end{equation}
then the proof is complete.
\end{proof}

\subsubsection{Proof of Theorem \ref{theorem:convergence_rate_scratch}}
Now, we are ready to present the proof of the main theorem with the above lemmas. Based on (\ref{eq:decompose}), since we have already bounded $\mathbb{E}_t[\mathcal C]$ in Lemma \ref{lemma:client_bound}, we now focus on bounding $\mathbb{E}_t[\mathcal S]$. 
For the server-side model, we have:

\begin{equation}\label{eq:server-decouple}
    \begin{aligned}
            \mathbb{E}_t[\mathcal S]&= \mathbb{E}_t\big[\nabla f_s(\boldsymbol\theta_\text{s}^t;\boldsymbol\theta_\text{c,:}^*)^T (\boldsymbol\theta_\text{s}^{t+1} - \boldsymbol\theta_\text{s}^t) + \frac{L}{2} \|\boldsymbol\theta_\text{s}^{t+1} - \boldsymbol\theta_\text{s}^t\|^2\big]\\
    =& \big\langle \nabla f_s(\boldsymbol\theta_\text{s}^t;\boldsymbol\theta_\text{c,:}^*), \mathbb{E}_t[\boldsymbol\theta_\text{s}^{t+1} - \boldsymbol\theta_\text{s}^t]\big\rangle + \frac{L}{2} \mathbb{E}_t\big[\|\boldsymbol\theta_\text{s}^{t+1} - \boldsymbol\theta_\text{s}^t\|^2\big]\\
    \overset{(a)}{=} &  
    {\big\langle\nabla f_s(\boldsymbol\theta_\text{s}^t;\boldsymbol\theta_\text{c,:}^*), - \eta_s\mathbb{E}_t\big[ \sum\nolimits_{i=1}^N \nabla f_s(\boldsymbol\theta_{s,i}^t;\boldsymbol\theta_\text{c,:}^t)\big]\big\rangle}
    + \frac{L\eta_s^2}{2} {\mathbb{E}_t\big[\big\| \sum\nolimits_{i=1}^N \nabla f_s(\boldsymbol\theta_{s,i}^t;\boldsymbol\theta_\text{c,:}^t)\big\|^2\big]},
    \end{aligned}
\end{equation}
where $(a)$ holds because of the update rule (\ref{Eq:ServerUpdate}) of the server-side model. 
At this part, we follow the same steps as in \cite{mu2025federated} to bound $\mathbb{E}_t[\mathcal S]$. So with additional Assumption \ref{assumption:distribution_drift}, based on the theoretical results of the server-side model, we can derive the following bound of $\mathbb{E}_t[\mathcal S]$:
\begin{equation}\label{eq:server_bound}
    \begin{aligned}
        \mathbb{E}_t[\mathcal{S}]\overset{(a)}{\leq}&  \eta_sG_s^2\sum\nolimits_{i=1}^{N} d_{c,i}^t - \frac{\eta_s(2N-1)}{4}\|\nabla f_s(\boldsymbol\theta_{s}^t)\|^2 + \frac{L}{2}N^2\eta_s^2G_s^2,
    \end{aligned}
\end{equation}
where $(a)$ holds if and only if $\eta_s\leq \frac{1}{NL}$, which means the term $(\frac{L\eta_s^2}{2}-\frac{\eta_s}{2N})\mathbb{E}_t[\|\sum\nolimits_{i=1}^{N}\nabla f_s(\boldsymbol\theta_{s,i}^t)\|^2]$ is non-positive.

Combining the bounds of $\mathbb{E}_t[\mathcal C]$ and $\mathbb{E}_t[\mathcal S]$, we have:
\begin{equation}
    \begin{aligned}
        \mathbb{E}_t[\mathcal{C} + \mathcal{S}]
        \leq& -\frac{\eta_c h}{4}\Vert\nabla  f_c(\boldsymbol{\theta}_c^t)\Vert^2 + \Phi_c(\eta_c) + \frac{\eta_sG_s^2}{2N} \sum\nolimits_{i=1}^{N} d_{c,i}^t  -\frac{\eta_s(2N-1)}{4}\|\nabla f_s(\boldsymbol\theta_{s}^t)\|^2+ \frac{L}{2}N^2\eta_s^2G_s^2,
    \end{aligned}
\end{equation}
with $\eta_c\leq\min\{\frac{1}{3Lh}, \frac{2}{NLh^2}, \frac{N}{72L}\}$ we have:
\begin{equation}
    \begin{aligned}
        &\mathbb E\big[f(\boldsymbol\theta_\text{g}^{t+1}) \big]\\
        \leq & f(\boldsymbol\theta_\text{g}^t) -\frac{\eta_c h}{4}\Vert\nabla f_c(\boldsymbol{\theta}_c^t)\Vert^2+ \Phi_c(\eta_c)  + \frac{\eta_sG_s^2}{2N} \sum\nolimits_{i=1}^{N} d_{c,i}^t -\frac{\eta_s(2N-1)}{4}\|\nabla f_s(\boldsymbol\theta_{s}^t)\|^2 + \frac{L}{2}N^2\eta_s^2G_s^2.
    \end{aligned}
\end{equation}
Rearranging the terms and applying $\min\{a,b\}(x+y) \le ax+by$, we have:
\begin{equation}
\begin{aligned}
        \Vert\nabla f(\boldsymbol{\theta}^t_\text{g})\Vert^2
        &\leq \frac{1}{\min\{ \frac{\eta_c h}{4},  \frac{\eta_s(2N-1)}{4}\}}\big(f(\boldsymbol\theta_\text{g}^t) - \mathbb E\big[f(\boldsymbol\theta_\text{g}^{t+1}) \big] + \Phi_c(\eta_c)+ \frac{\eta_s}{2N} \sum\nolimits_{i=1}^{N} G_s^2d_{c,i}^t+ \frac{L}{2}N^2\eta_s^2G_s^2\big),\\
         &\eta_c\leq \min\{\frac{1}{3Lh}, \frac{2}{NLh^2}, \frac{N}{72L}\}, \forall t\in[T].
\end{aligned}
\end{equation}

Taking the full expectation, summing over $t=1,\dots,T$, and defining a coupled step size $\eta \triangleq \frac{\eta_c h}{4} = \frac{\eta_s(2N-1)}{4}$ to balance the updates, we simplify the bound to:
\begin{equation}
\begin{aligned}
    \min_{t\in[T]}\mathbb{E}\big[\Vert\nabla f(\boldsymbol{\theta}^t_\text{g})\Vert^2\big]
    \leq& \frac{f(\boldsymbol\theta_\text{g}^1)-f^*}{\eta T} + \mathcal C^{\prime}_1 \eta + C_2 \mu^2 + \frac{C_3}{N^2},
\end{aligned}
\end{equation}
where $\mathcal C^{\prime}_1 \triangleq \mathcal{O}\big(\frac{LG_s^2}{N} + \frac{LdG_c^2}{hNB} + \frac{L\sigma_c^2}{hN}\big)$ denotes the variance term, $C_2 \triangleq \mathcal{O}\big(\frac{d^2L^2}{B}\big)$ represents the ZO bias, and $C_3$ groups the residual drift terms.
By selecting the step sizes and smoothing parameter asthen we should have
$\eta = \mathcal{O}(\sqrt{\frac{hNB}{dT}})$,
$\eta_c=\mathcal{O}(\sqrt{\frac{NB}{dhT}})$,
$\eta_s=\mathcal{O}(\sqrt{\frac{hB}{dNT}})$,
and $\mu=\mathcal{O}((dhNBT)^{-1/4})$. Consequently, the convergence rate is bounded by:
\begin{equation}
    \min_{t\in[T]}\mathbb{E}\big[\Vert\nabla f(\boldsymbol{\theta}^t_\text{g})\Vert^2\big]
    \leq \mathcal{O}\big(\sqrt{\frac{d}{hNBT}}\big) + \mathcal{O}\big(\frac{1}{\sqrt{dhNBT}}\big).
\end{equation}
This completes the proof.

\subsection{Proof of Theorem \ref{theorem:low_rank}}
In this section, we consider the convergence behavior from the perspective of the language model fine-tuning situation. Since the loss landscape of deep learning lies in a very low-dimensional subspace, where the Hessian of the loss has a remarkably low effective rank, we can leverage this property to analyze the convergence rates more effectively. 
\begin{lemma}[\textbf{Client-side bound with low effective-rank}]
\label{lemma:li2024}
Under Assumptions \ref{assumption:smoothness}--\ref{assumption:low_rank}, and drawing $\boldsymbol{u}_i^t$ uniformly from the sphere of radius $\sqrt{d}$, the client-side contribution is bounded by:
\begin{equation}
    \begin{aligned}
        \mathbb{E}[\mathcal{C}] \leq & -\frac{\eta_c}{4}\|\nabla f_c(\boldsymbol{\theta}_c^t)\|^2 + \frac{\eta_c\mu^2L^2}{8}(d+3)^3 \\
        & + \eta_c^2L^4\mu^2d^3 + \eta_c^3L^2 \rho^2 + \frac{\eta_c^2 L \rho}{N}(\sigma^2+G_s^2),
    \end{aligned}
\end{equation}
where we define the geometric factor $\rho \triangleq 1+\frac{d\kappa+d-2}{d+2}$.
\end{lemma}

\begin{proof}
Following the analysis framework established in Theorem 2 of \cite{li2024achieving}, specifically adapting their Eq.~(69) to our notation, we obtain the initial bound:
\begin{equation}\label{eq:initial_consistent}
    \begin{aligned}
        \mathbb{E}[\mathcal{C}]\leq& -\frac{\eta_c}{2}\|\nabla f_c(\boldsymbol{\theta}_c^t)\|^2 + \frac{\eta_c\mu^2L^2}{8}(d+3)^3+\eta_c^2L^4\mu^2d^3\\&
        +\eta_c^2L\rho \big(\|\nabla f_c(\boldsymbol{\theta}_c^t)\|+\frac{\sigma^2+G_s^2}{N}\big),
    \end{aligned}
\end{equation}
where $\rho$ is defined as in Lemma~\ref{lemma:client_bound}. Our goal is to absorb the linear gradient norm term into the negative quadratic term. First, we isolate the terms involving $\|\nabla f_c(\boldsymbol{\theta}_c^t)\|$ from the constant terms:
\begin{equation}\label{eq:terms_of_interest}
    \text{RHS} \leq \underbrace{-\frac{\eta_c}{2}\|\nabla f_c(\boldsymbol{\theta}_c^t)\|^2 + \eta_c^2 L \rho \|\nabla f_c(\boldsymbol{\theta}_c^t)\|}_{\mathcal{T}_{\text{grad}}} + \mathcal C^{\prime}_1,
\end{equation}
where $\mathcal C^{\prime}_1$ collects the terms independent of the gradient norm:
\begin{equation}
    \mathcal C^{\prime}_1 \triangleq \frac{\eta_c\mu^2L^2}{8}(d+3)^3+\eta_c^2L^4\mu^2d^3 + \frac{\eta_c^2 L \rho}{N}(\sigma^2+G_s^2).
\end{equation}
To bound the linear term in $\mathcal{T}_{\text{grad}}$, we apply Young's inequality, $ab \le \frac{\delta}{2}a^2 + \frac{1}{2\delta}b^2$, for a free parameter $\delta > 0$. By setting $a = \eta_c L \rho$ and $b = \eta_c \|\nabla f_c(\boldsymbol{\theta}_c^t)\|$, we have:
\begin{equation}\label{eq:youngs_step}
    \begin{aligned}
        \eta_c^2 L \rho \|\nabla f_c(\boldsymbol{\theta}_c^t)\| \le \frac{\delta}{2} (\eta_c L \rho)^2 + \frac{1}{2\delta} \eta_c^2 \|\nabla f_c(\boldsymbol{\theta}_c^t)\|^2.
    \end{aligned}
\end{equation}
Substituting \eqref{eq:youngs_step} back into $\mathcal{T}_{\text{grad}}$ yields:
\begin{equation}
    \mathcal{T}_{\text{grad}} \le \big( -\frac{\eta_c}{2} + \frac{\eta_c^2}{2\delta} \big) \|\nabla f_c(\boldsymbol{\theta}_c^t)\|^2 + \frac{\delta \eta_c^2 L^2 \rho^2}{2}.
\end{equation}
We strategically choose $\delta = 2\eta_c$ to ensure the coefficient of the squared gradient norm becomes $-\frac{\eta_c}{4}$:
\begin{equation}
    -\frac{\eta_c}{2} + \frac{\eta_c^2}{2(2\eta_c)} = -\frac{\eta_c}{2} + \frac{\eta_c}{4} = -\frac{\eta_c}{4}.
\end{equation}
With this choice, the additional term from Young's inequality becomes $\frac{(2\eta_c)\eta_c^2 L^2 \rho^2}{2} = \eta_c^3 L^2 \rho^2$. Combining this with $\mathcal C^{\prime}_1$ in \eqref{eq:terms_of_interest} completes the proof.
\end{proof}

\subsubsection{Proof of Theorem \ref{theorem:low_rank}}
Same as Theorem 1, the analysis can be naturally divided into two parts: client-side optimization and server-side optimization. 
\begin{proof}
Combining the bounds of $\mathbb{E}_t[\mathcal C]$ from Lemma \ref{lemma:li2024} and $\mathbb{E}_t[\mathcal S]$ same as (\ref{eq:server_bound}), we have:
\begin{equation}
    \begin{aligned}
        &\mathbb{E}_t[\mathcal{C} + \mathcal{S}]=\mathbb{E}[\mathcal{C} + \mathcal{S}]\\
        \leq& -\frac{\eta_c }{4}\Vert\nabla f_c(\boldsymbol{\theta}_c^t)\Vert^2 + \Phi_c'(\eta_c)+ \frac{L}{2}N^2\eta_s^2G_s^2\\
        &+ \frac{\eta_sG_s^2}{2N} \sum\nolimits_{i=1}^{N} d_{c,i}^t  -\frac{\eta_s(2N-1)}{4}\|\nabla f_s(\boldsymbol\theta_{s}^t)\|^2,
    \end{aligned}
\end{equation}
where $\Phi_c'$ is defined as:
\begin{equation}
    \begin{aligned}
        \Phi_c' = &\frac{\eta_c\mu^2L^2}{8}(d+3)^3+\eta_c^2L^4\mu^2d^3 + \eta_c^3L^2\big(1+\frac{d\kappa+d-2}{d+2}\big)^2 \\&+ \eta_c^2L\big(1+\frac{d\kappa+d-2}{d+2}\big)\frac{1}{N}(\sigma^2+G_s^2) .
    \end{aligned}
\end{equation}
Following the proof methodology of Theorem 4.1, we derive the following bound on the global gradient norm:
\begin{equation}
\begin{aligned}
    \|\nabla f(\boldsymbol{\theta}^t_\text{g})\|^2
    \leq& \frac{1}{\eta_{\min}}\big(  f(\boldsymbol\theta_\text{g}^t) - \mathbb E\big[f(\boldsymbol\theta_\text{g}^{t+1}) \big] + \Phi_c'(\eta_c) \\
    &+ \frac{\eta_s G_s^2}{2N} \sum_{i=1}^{N} d_{c,i}^t+ \frac{1}{2}LN^2\eta_s^2G_s^2 \big),
\end{aligned}
\end{equation}
where $\eta_{\min} \triangleq \min\{ \frac{\eta_c}{4}, \frac{\eta_s(2N-1)}{4}\}$. Taking the full expectation and summing over $t=1,\dots,T$ under Assumption \ref{assumption:distribution_drift}, we obtain:
\begin{equation}
\begin{aligned}
    \min_{t\in[T]}\mathbb{E}\big[\|\nabla f(\boldsymbol{\theta}^t_\text{g})\|^2\big]
    \leq& \frac{f(\boldsymbol\theta_\text{g}^1)-f(\boldsymbol\theta_\text{g}^{T+1})}{\eta_{\min} T} + \frac{\Phi_c'(\eta_c)}{\eta_{\min}} \\
    &+ \frac{\eta_s G_s^2 \sum_{i=1}^{N} d_{c,i}^t}{2N \eta_{\min}} + \frac{LN^2\eta_s^2G_s^2}{2\eta_{\min}} .
\end{aligned}
\end{equation}
To ensure consistent convergence rates across client and server updates, we couple the step sizes by setting $\eta \triangleq \eta_c/4 = (2N-1)\eta_s/4$. The bound can then be reorganized by the order of $\eta$:
\begin{equation}\label{eq:58_detailed}
\begin{aligned}
    &\min_{t\in[T]}\mathbb{E}\big[\|\nabla f(\boldsymbol{\theta}^t_\text{g})\|^2\big] \\
    \leq& \frac{f(\boldsymbol\theta_\text{g}^1)-f(\boldsymbol\theta_\text{g}^{T+1})}{\eta T} + \frac{\mu^2L^2}{2}(d+3)^3+  \frac{2G_s^2}{N(2N-1)}\sum\nolimits_{i=1}^{N} d_{c,i}^t \\
    &+ \eta \big[ \frac{8LN^2G_s^2}{(2N-1)^2} + 16L^4\mu^2d^3 + \frac{16L}{N}\big(1+\frac{d\kappa+d-2}{d+2}\big)(\sigma^2+G_s^2) \big] \\
    &+ \eta^2 \big[ 64L^2\big(1+\frac{d\kappa+d-2}{d+2}\big)^2 \big].
\end{aligned}
\end{equation}
To obtain an informative rate, we analyze the structure of the dominant terms. In typical federated learning settings, the coefficient of the leading $\mathcal{O}(\eta)$ variance term scales with the condition number $\kappa$, client count $N$, and batch size $B$, approximately behaving as $\mathcal{O}(\kappa/(NB))$.

Optimizing the bound requires balancing the initialization term $\mathcal{O}(1/(\eta T))$ with this dominant $\mathcal{O}(\eta)$ term. The condition $\frac{1}{\eta T} \asymp \eta \frac{\kappa}{NB}$ implies $\eta \propto \sqrt{NB/(\kappa T)}$. Accordingly, we set the parameters as follows:
\begin{equation}\label{eq:59_detailed}
    \eta, \eta_c = \mathcal{O}\big(\sqrt{\frac{NB}{\kappa T}}\big), \quad \eta_s = \mathcal{O}\big(\sqrt{\frac{B}{N\kappa T}}\big),
\end{equation}
and the smoothing parameter $\mu \leq \kappa^{1/4}(NT)^{-1/4}(d+3)^{-3/2}$. Substituting these choices into \eqref{eq:58_detailed}, we obtain the final convergence rate:
\begin{equation}\label{eq:60_detailed}
    \min_{t\in[T]}\mathbb{E}\big[\|\nabla f(\boldsymbol{\theta}^t_\text{g})\|^2\big] \leq \mathcal{O}\big(\sqrt{\frac{\kappa}{NBT}}\big) + \mathcal{O}\big(\frac{1}{T}\big) + C_{err},
\end{equation}
where the $\mathcal{O}(1/T)$ term arises from the $\eta^2$ components, and $C_{err}$ denotes the constant error floor:
\begin{equation}
    C_{err} \triangleq \frac{2}{\delta}\big[\frac{2G_s^2}{N(2N-1)}\Delta + \frac{\mu^2L^2}{2}(d+3)^3\big].
\end{equation}
This indicates that the algorithm converges to a neighborhood of the optimum determined by the drift limit $\Delta$ and the smoothing radius $\mu$.
\end{proof}

\section{Evidence for Low Rank Assumption}\label{App:low-rank}
In this section, we provide empirical and literature evidence that the low-effective-rank phenomenon is broadly present in deep model training (including CNN training and LM fine-tuning), which motivates the low-rank assumption used in our analysis.
First, we validated the low-effective rank assumption using a modified ResNet-18 on CIFAR-10. We estimated the Hessian eigenvalue density via the stochastic Lanczos algorithm \cite{golub1969calculation}, following the methodology of \cite{ghorbani2019investigation}. As shown in Fig.~\ref{fig:low_rank}, the resulting distribution, heavily concentrated at zero, suggests that the low-rank structure is an intrinsic property of the optimization landscape. The same empirical evidence can be seen in Appendix C.3.1 of \cite{li2024achieving}.

This observation extends to the regime of LMs, particularly during the fine-tuning phase. Recent works, such as GaLore \cite{zhao2024galore}, have provided robust evidence that while pre-training may necessitate high-rank updates, the weight modifications required for fine-tuning naturally reside in a low-rank subspace. This intrinsic low-dimensionality is a critical factor explaining the success of ZO optimization methods in this domain \cite{malladi2023fine}. 
It theoretically justifies why methods like MeZO \cite{malladi2023fine} can achieve competitive performance with LoRA updates, as they effectively navigate this low-rank manifold.

\begin{figure}[H]
    \centering
    \includegraphics[width=0.5\linewidth]{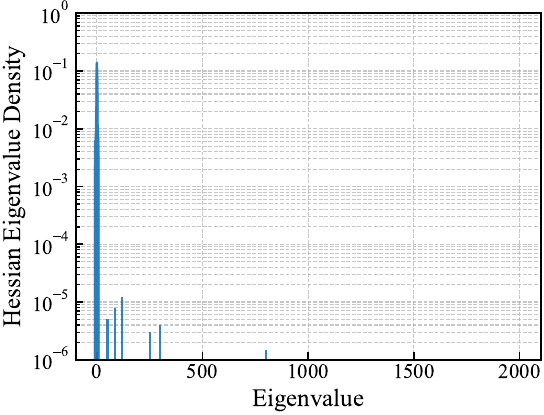}
    \caption{Hessian eigenvalue distribution with training a custom ResNet on the CIFAR-10 dataset.}
    \label{fig:low_rank}
\end{figure}
\bibliographystyle{IEEEtran}
\bibliography{refs}

\end{document}